\PassOptionsToPackage{table,xcdraw}{xcolor}

\documentclass{article}

\usepackage{microtype}
\usepackage{graphicx}
\usepackage{subcaption}
\usepackage{multirow}
\usepackage{booktabs}
\usepackage{tabularx}
\usepackage{color}

\def\red{\color{red}}

\usepackage{hyperref}

\usepackage[preprint]{icml2024}

\usepackage{amsmath}
\usepackage{amssymb}
\usepackage{mathtools}
\usepackage{amsthm}

\usepackage[capitalize,noabbrev]{cleveref}

\theoremstyle{plain}
\newtheorem{theorem}{Theorem}[section]
\newtheorem{proposition}[theorem]{Proposition}
\newtheorem{lemma}[theorem]{Lemma}

\newtheorem{example}{Example}[section]
\newtheorem{fact}{Fact}[section]
\theoremstyle{definition}
\newtheorem{definition}[theorem]{Definition}
\newtheorem{assumption}[theorem]{Assumption}
\theoremstyle{remark}
\newtheorem{remark}[theorem]{Remark}

\input{math_commands}
\renewcommand{\leq}{\leqslant}
\renewcommand{\le}{\leqslant}
\renewcommand{\geq}{\geqslant}
\renewcommand{\ge}{\geqslant}

\newcommand{\est}{\operatorname{est}}
\newcommand{\full}{\operatorname{full}}
\newcommand{\adv}{\operatorname{adv}}

\newcommand{\clip}{\operatorname{clip}}

\usepackage{amsthm}
\usepackage{amsmath}
\usepackage{amssymb}
\usepackage{thm-restate}

\usepackage{babel}

\icmltitlerunning{Optimism-then-NoReget}

\usepackage{listings}
\usepackage{soul}
\usepackage{comment}

\definecolor{codegreen}{rgb}{0,0.6,0}
\definecolor{codegray}{rgb}{0.5,0.5,0.5}
\definecolor{codepurple}{rgb}{0.58,0,0.82}
\definecolor{backcolour}{rgb}{0.95,0.95,0.92}

\begin{document}

\twocolumn[
    \icmltitle{Optimistic Thompson Sampling for No-Regret Learning in Unknown Games}

    \icmlsetsymbol{equal}{*}

    \begin{icmlauthorlist}
        \icmlauthor{Yingru Li}{equal,cuhksz,sribd}
        \icmlauthor{Liangqi Liu}{equal,cuhksz,sribd}
        \icmlauthor{Wenqiang Pu}{sribd}
        \icmlauthor{Hao Liang}{cuhksz}
        \icmlauthor{Zhi-Quan Luo}{cuhksz,sribd}
    \end{icmlauthorlist}

    \icmlaffiliation{cuhksz}{The Chinese University of Hong Kong, Shenzhen, China}
    \icmlaffiliation{sribd}{Shenzhen Research Institute of Big Data}

    \icmlcorrespondingauthor{Yingru Li}{yingruli@cuhk.edu.cn}
    \icmlcorrespondingauthor{Wenqiang Pu}{wenqiangpu@cuhk.edu.cn}

    \icmlkeywords{Machine Learning, ICML}

    \vskip 0.3in
]

\printAffiliationsAndNotice{\icmlEqualContribution} %

\begin{abstract}

This work tackles the complexities of multi-player scenarios in \emph{unknown games}, where the primary challenge lies in navigating the uncertainty of the environment through bandit feedback alongside strategic decision-making. We introduce Thompson Sampling (TS)-based algorithms that exploit the information of opponents' actions and reward structures, leading to a substantial reduction in experimental budgets---achieving over tenfold improvements compared to conventional approaches. Notably, our algorithms demonstrate that, given specific reward structures, the regret bound depends logarithmically on the total action space, significantly alleviating the curse of multi-player. Furthermore, we unveil the \emph{Optimism-then-NoRegret} (OTN) framework, a pioneering methodology that seamlessly incorporates our advancements with established algorithms, showcasing its utility in practical scenarios such as traffic routing and radar sensing in the real world.

\end{abstract}
\begin{figure}[!htbp]
    \centering\includegraphics[width=0.8\linewidth]{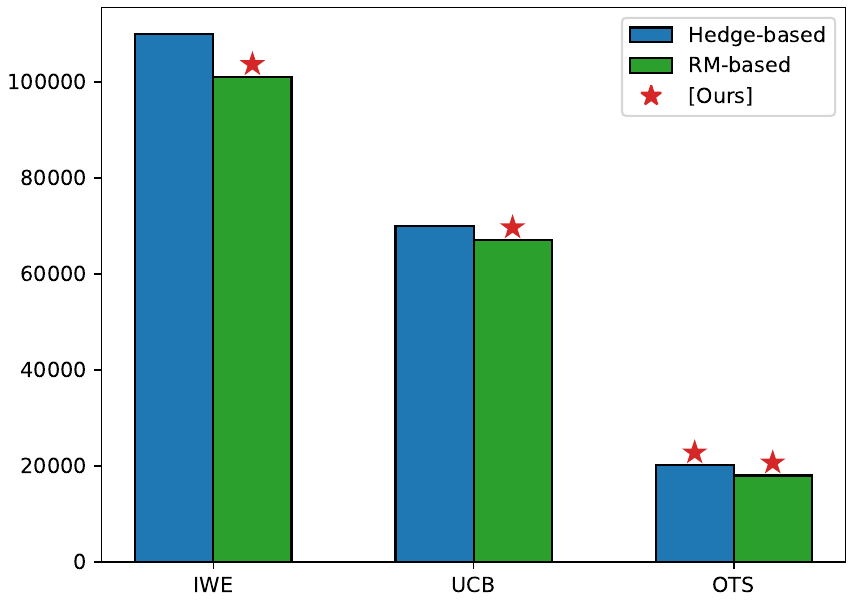}
    \caption{
        \textbf{Reduction on experimental budgets} by our methods in terms of \# samples required for various algorithms to reach the same level in average regret $(10^{-1})$ under the $50\times50$ matrix game setups ( \cref{sec:rmg}).
        IWE, UCB and OTS correspond to \underline{I}mportance \underline{W}eighted \underline{E}stimator, \underline{U}pper \underline{C}onfidence \underline{B}ound and \underline{O}ptimistic \underline{T}hompson sampling.
        RM corresponds to \underline{R}egret \underline{M}atching.
        IWE-RM, UCB-RM, OTS-Hedge and OTS-RM are our proposed algorithms.
        The compared baselines include
        IWE-Hedge (known as Exp3 ~\citep{auer2002nonstochastic} while we implement all Exp3 variants~\citep{stoltz2005incomplete,kocak2014efficient,lattimore2020bandit} and select the best as baseline.) and
        UCB-Hedge (known as GP-MW in~\cite{sessa2019no}).
    }
    \label{fig:hist}
    \vspace{-3.5mm}
\end{figure}

\section{Introduction}
\label{sec:intro}
Many real-world problems in economics~\citep{fainmesser2012community}, sociology~\citep{skyrms2009dynamic}, transportation~\citep{leblanc1975algorithm}, politics~\citep{ordeshook1986game}, signal processing~\citep{song2011mimo}, and other fields~\citep{fudenberg1991game} can be described as \textit{unknown games}, where each player only observes their opponents’ actions and the noisy rewards associated with their own selected actions (referred to as \textit{bandit feedback}). The goal of each player is to maximize their individual reward, and the only way to achieve this goal is by repeatedly playing the game and learning its structure from the observed rewards. The challenge in unknown games is how to efficiently learn from bandit feedback. Celebrated no-regret learning algorithms, such as Hedge~\citep{freund1997decision} and regret-matching~\citep{hart2000simple}, ensure sublinear regret guarantees under the \textit{full information} setting, where the rewards of all actions at each round are observable. However, these algorithms cannot handle problems with bandit feedback and face the curse of multi-player: the complexity scales exponentially as the number of players increases.
We aim to develop efficient algorithms learning to play repeated unknown games by leveraging information on opponents' actions and reward structures. We highlight the contributions:
\begin{itemize}
\vspace{-2mm}
\item \emph{\textbf{Effective algorithms}}: The vanilla TS indeed fails in a specific class of unknown games as we demonstrated. To overcome this drawback, we introduce an optimistic variant of TS (a.k.a. OTS) combined with appropriate full information adversarial bandit algorithms. We show that OTS can fix the divergence issue with the help of Gaussian anti-concentration behavior.
Empirically, our proposed methods reduce the experimental budgets \emph{more than an order of magnitude} in two real-world problems - two-player radar sensing (\cref{fig:radar_jammer}) and multi-player traffic routing problem (\cref{fig:traffic}) - and the random matrix game (\cref{fig:hist,fig:rdnmtx01}).
Meanwhile, our proposed methodology effectively mitigates the curse of multi-player: OTS-type algorithms are capable of solving the traffic routing problem with hundreds of decision-makers (\cref{fig:traffic}), resulting in minimal congestion.
\item \emph{\textbf{Theoretical advantages}}: A general information-theoretic regret bound is provided and sublinear regret bounds of all proposed algorithms are established accordingly (Table~\ref{tab:regret}). Our analysis highlights that using (1) the information of opponent's actions and (2) the underlying reward structure can help resolve the curse of multi-player. For structured reward functions, our algorithms achieve regret bounds that depend logarithmically on the size of the action space. In contrast, algorithms that rely only on bandit feedback suffer from the curse of multi-player.
\item \emph{\textbf{Unified framework}}: An Optimism-then-NoRegret (OTN) learning framework for unknown games is also introduced. This framework encompasses various vanilla game algorithms as special cases, and several efficient algorithms can be developed under this framework, including upper confidence bound (UCB) based algorithms and TS-based algorithms. Notably, the proposed OTS-RM algorithm from this framework achieves the best performance in all experiments.
\end{itemize}

\begin{table*}[!htbp]
\centering
    \small
    \begin{tabular}{@{}cc@{\hspace{6pt}}c@{\hspace{6pt}}c@{\hspace{6pt}}c@{\hspace{6pt}}c@{}} 
    \toprule
    Feedback        &  & Full & Bandit & Bandit + Actions & Bandit + Actions \\ \midrule
    Imagined Reward &  & --                     & IWE      & UCB       & OTS \cellcolor[HTML]{C1FEC0}              \\
    \multirow{2}{*}{No-Regret Update} &
    Hedge &
    $\mathcal{O} \big( \sqrt{T  \log \abs{\actions}} \big)$ &
    $\mathcal{O} \big( \sqrt{T \actions  \log \abs{\actions}} \big)$ &
    $\mathcal{O} \big( \sqrt{T  \log \abs{\actions}} + \sqrt{\gamma_T\beta T}\big)$ & \cellcolor[HTML]{C1FEC0}
    $\mathcal{O} \big( \sqrt{T  \log \abs{\actions}} + \sqrt{\gamma_T\beta T}\big)$ \\
    &
    RM &
    $\mathcal{O} \big( \sqrt{T \abs{\actions}} \big)$ & \cellcolor[HTML]{C1FEC0}
    $\mathcal{O} \big( T^{2/3} \abs{\actions}^{2/3} \big)$ & \cellcolor[HTML]{C1FEC0}
    $\mathcal{O} \big( \sqrt{T \abs{\actions}} + \sqrt{\gamma_T\beta T}\big)$  &  \cellcolor[HTML]{C1FEC0}
    $\mathcal{O} \big( \sqrt{T \abs{\actions}} + \sqrt{\gamma_T\beta T}\big)$ \\ \bottomrule
    \end{tabular} 
    \caption{
    Comparison on the regret bounds.
    $\actions$ is the action set of the agent and $\mathcal{B}$ is the action set of opponent. $\gamma_T$ is the maximum information gain (\cref{def:max_infor_gain}) and $\beta=\log (\abs{\actions}T)$. The cells with green color are our results.}
    \label{tab:regret}
    \vspace{-3mm}
\end{table*}

\subsection{Related works}
\vspace{-1mm}
\paragraph{Adversarial Bandits.} In the full-information setting, multiplicative-weights (MW) algorithms such as Hedge~\citep{freund1997decision} achieve optimal regret for adversarial bandit problems. However, full information feedback, requiring perfect game knowledge, is unrealistic in many applications. In the challenging bandit feedback setting, the Exp3 algorithm~\cite{auer2002nonstochastic,stoltz2005incomplete,kocak2014efficient,neu2015explore,lattimore2020bandit} is notable for utilizing an importance-weighted estimator to construct the reward vector.

\vspace{-3mm}
\paragraph{Learning in Games.} A series of works~\citep{daskalakis2011near, syrgkanis2015fast,chen2020hedging,hsieh2021adaptive} have studied no-regret learning algorithms in games, with regret matching~\citep{hart2000simple, hart2001reinforcement} being another prevalent approach. A variation, $\text{regret-matching}^{+}(\text{RM}^{+})$~\citep{tammelin2014solving}, has been shown to lead to significantly faster convergence in practice. These online algorithms treat opponents as part of the environment, thereby reducing the unknown repeated game to a bandit problem. In this adversarial and adaptive environment, reward functions vary over different time steps \citep{cesa2006prediction}.

\paragraph{Structure \& Opponent Awareness.} Prior literature often overlooks the potential to exploit the reward structure in repeated games and the observability of opponents' actions. This oversight persists despite the scenario's relevance to numerous applications~\cite{o2021matrix,sessa2019no}. \cite{o2021matrix} compares the received reward to the Nash value, proposing UCB and K-learning variants to minimize Nash regret. In contrast, our work aims to exploit the opponent's strategy, introducing a focus on adversarial regret—a metric discussed further in \cref{sec:addtional-related}. This approach differs fundamentally from seeking only to achieve the Nash value. \cite{sessa2019no} also prioritizes adversarial regret, employing a Gaussian Process to exploit correlations among game outcomes and achieve a kernel-dependent regret bound. This bound includes the factor $\gamma_T$, derived from an UCB-type algorithm.
Thompson sampling (TS) and its variants~\citep{russo2018tutorial,vaswani20old} represent a strong alternative to UCB-type algorithms in the context of reward structure-aware bandit literature. Recent works by \citet{zhang2022feel,agarwal2022model} have explored the incorporation of optimism into TS through an optimistic prior, albeit facing computational tractability issues. Other research~\citep{li2022hyperdqn,li2023efficient,li2024hyperagent} addresses these computational challenges in TS for large-scale, complex environments in single-agent setups. However, evidence supporting the effectiveness of these randomized exploration methods in our specific setting of unknown games remains limited.

Our proposed methodologies not only address the computational tractability issues inherent in optimistic TS but also introduce a novel perspective on learning in unknown game environments. By focusing on adversarial regret, we provide a more nuanced understanding of how players can strategically navigate these games to their advantage. This shift in focus from merely achieving Nash equilibrium to exploiting strategic opportunities represents a significant departure from traditional approaches.

This paper is organized as follows: \Cref{sec:formulations} introduces the fundamental protocols and notations used in the repeated bandit game. \Cref{sec:algorithm} describes the proposed OTS-type and UCB-RM algorithms, as well as the corresponding OTN framework. \Cref{sec:analysis} provides an analysis of the regret associated with these methods. \Cref{sec:experiments} details the experimental results, showcasing the effectiveness of our approaches. Finally, Section \ref{sec:conlcu} concludes the paper.

\section{Repeated bandit game}
\label{sec:formulations}
To simplify the exposition, we consider a two-player game scenario involving Alice and Bob. However, our results can be straightforwardly extended to multiplayer games by treating all other players as an \emph{abstract} player.

\vspace{-0.5em}
\paragraph{Protocol.}
Consider a repeated game between Alice and her opponent, Bob, where the action index sets for Alice and Bob are denoted by $\mathcal{A} = \{1, \ldots, |\mathcal{A}| \}$ and $\mathcal{B} = \{1, \ldots, |\mathcal{B}| \}$, respectively.\footnote{We use the shorthand $|\mathcal{A}|$ to denote the cardinality of Alice's action set, and similarly $|\mathcal{B}|$ for Bob's.}
At each time $t=0,1, \ldots$, Alice selects an action $A_{t} \in \mathcal{A}$, and Bob simultaneously selects an action $B_t \in \mathcal{B}$. The mean reward for each action pair $(a,b)$ is $f_{\theta}(a,b)$, where $\theta \in \Theta$ is a model parameter unknown to the players. In the \textit{bandit feedback} scenario, Alice only observes a noisy version of the mean reward associated with the selected action pair $(A_{t}, B_t)$:
\[ Y_{t+1, A_{t}, B_t} = f_{\theta}(A_{t}, B_t) + \eta_t, \]
where $\eta_t$ is an i.i.d. noise sequence. Under the \textit{full information} setting, Alice can observe the mean reward vector\footnote{Details of the full information feedback protocol are presented in~\cref{app:game}} $\mathbf{r}_t = (f_{\theta}(a, B_t))_{a \in \mathcal{A}}$ associated with each action $a \in \mathcal{A}$.
Alice's experience up to time $t$ is encoded by the history $H_{t} = \{(A_{0}, B_0, Y_{1, A_{0}, B_0}), \ldots, (A_{t-1}, B_{t-1}, Y_{t, A_{t-1}, B_{t-1}}) \}$.

\vspace{-0.5em}
\paragraph{Algorithm.}
An algorithm $\pi^{\text{alg}} = (\pi_t)_{t \in \mathbb{N}}$ employed by Alice is a sequence of deterministic functions, where each $\pi_t(H_t)$ specifies a probability distribution over the action set $\mathcal{A}$ based on the history $H_t$. Alice's action $A_t$ is sampled from the distribution $\pi_t(H_t)$, i.e., $\prob(A_t \in \cdot | H_t) = \pi_t(\cdot)$.

The above description of the bandit game encompasses various game forms based on the structure of the mean function $f_{\theta}(a, b)$. Several representative game forms, such as the matrix game, linear game, and kernelized game, are summarized in Table~\ref{tab:game} (see Appendix~\ref{app:game}).

\subparagraph{Reward and Performance Metric.}
Alice maintains a reward function $\mathcal{R}: \mathbb{R} \mapsto [0, 1]$ that maps the observations to a bounded value, i.e., $R_{t+1, A_t, B_t} = \mathcal{R}(Y_{t+1, A_t, B_t})$. The objective for Alice is to maximize her expected reward $\sum_{t=0}^{T-1} \mathbb{E}[R_{t+1, A_t, B_t} | \theta]$ over some duration $T$, irrespective of Bob's fixed action sequence $B_{0:T}$. By treating Bob as the adversarial environment, the best action $A^*$ in hindsight is $A^* = \arg\max_{a \in \mathcal{A}} \sum_{t=0}^{T-1} \mathbb{E}[R_{t+1, a, B_t} | \theta]$, and the $T$-period \emph{adversarial regret} $\mathcal{R} (T, \pi^{\text{alg}}, B_{0:T}, \theta)$ is defined by
\begin{align}
\label{eq:adv_regret}
    \sum_{t=0}^{T-1} \mathbb{E}[R_{t+1, A^*, B_t} - R_{t+1, A_t, B_t} | \theta],
\end{align}
where the expectation is taken over the randomness in the actions $A_t$ and the rewards $R_{t+1, A_t, B_t}$. However, this adversarial regret is not a suitable metric under our game setting since it depends on Bob's specific action sequence $B_{0:T}$. We adopt the worst-case regret as the metric. An algorithm $\pi_{\text{alg}}$ is considered \emph{No-Regret} for Alice if, for any $B_{0:T}$, Alice suffers only sublinear regret, i.e.,
\begin{align*}
    \mathcal{R}^*(T, \pi^{\text{alg}}) = \sup_{B_{0:T} \in \mathcal{B}^T} \mathcal{R}(T, \pi^{\text{alg}}, B_{0:T}) = o(T),
\end{align*}
omitting $\theta$ for simplicity of notation.

\section{Optimism-then-NoRegret learning}
\label{sec:algorithm}
\subsection{Review of Full Information Feedback}
We start by providing a brief overview of the full information feedback setting, in which Alice can observe the mean rewards $r_t(a) = f_{\theta}(a, B_t)$ for all actions $a \in \mathcal{A}$. At time $t$, Alice picks action $A_t \sim P_{X_t}$, where $P_X(i) = \frac{X_i}{\sum_{j \in \mathcal{A}} X_j}$. Full-information adversarial bandit algorithms, such as Hedge~\citep{freund1997decision} and Regret Matching (RM)~\citep{hart2000simple}, can be used to update $X_t$ to $X_{t+1} = g_t(X_t, r_t)$,
\begin{align*}
\begin{cases}
  \text{Hedge:} & g_{t, a}(X_t, r_t) = X_{t, a} \exp\left(\eta_t r_t(a)\right),\\
  \text{RM:} & g_{t, a}(X_t, r_t) = \max\left(0, \sum\limits_{s=0}^t r_t(a) - r_t(A_s)\right),
\end{cases}
\end{align*}
where $g_t(\cdot): \mathbb{R}_+^{\mathcal{A}} \times \mathbb{R}_+^{\mathcal{A}} \mapsto \mathbb{R}_+^{\mathcal{A}}$. The full information adversarial regret of an algorithm $\text{adv}$ for a reward sequence $(r_t)_t$ is defined as
\begin{align*}
  \mathcal{R}_{\text{full}}(T, \text{adv}, (r_t)_t) = \max_{a \in \mathcal{A}} \mathbb{E}\left[\sum_{t=0}^{T-1} r_t(a) - r_t(A_t)\right].
\end{align*}
The worst-case regret is defined as
\begin{align}
\label{def:regret_full}
  \mathcal{R}_{\text{full}}(T, \text{adv}) = \max_{(r_t)_t} \mathcal{R}_{\text{full}}(T, \text{adv}, (r_t)_t).
\end{align}
Since $r_t(a) = f_{\theta}(a, B_t)$, the adversarial regret in \cref{eq:adv_regret} can be reformulated as full-information adversarial regret. For Hedge and RM, their worst-case regrets can be bounded as $\mathcal{R}_{\text{full}}(T, \text{Hedge}) = \mathcal{O}(\sqrt{T \log |\mathcal{A}|})$ and $\mathcal{R}_{\text{full}}(T, \text{RM}) = \mathcal{O}(\sqrt{T|\mathcal{A}|})$.

\subsection{Bandit Feedback}
In the bandit feedback setting, Alice can only observe a noisy version of the reward for the action she selects. We propose a framework that combines an optimism algorithm for stochastic bandits with a no-regret algorithm for full information adversarial bandits. First, we construct a sequence of surrogate full information feedback $\tilde{R}_{t} \in \mathbb{R}^{\mathcal{A}}$ in an optimistic sense, which we refer to as the \emph{imagined reward}. Specifically, we use an optimistic estimation algorithm $E$ to construct reward vector $\tilde{R}_{t+1} = E(H_{t+1}, Z_{t+1}) \in \mathbb{R}^{\mathcal{A}}.$
Then, we apply a no-regret update rule $g_t$ to update the sampling distribution with $\tilde{R}_{t+1}$ as $X_{t+1} = g_t(X_t, \tilde{R}_{t+1}).$
This procedure is described in Algorithm~\ref{alg:ETN}, termed Optimism-then-NoRegret (OTN).
\begin{algorithm}
  \caption{Optimism-then-NoRegret (OTN)}
  \label{alg:ETN}
  \begin{algorithmic}[1]
    \STATE Initialize $X_1$
    \FOR{round $t = 0, 1, \ldots, T-1$}
    \STATE Sample action $A_t \sim P_{X_t}$
    \STATE Observe opponent's action $B_t$ and noisy bandit feedback $R_{t+1, A_t, B_t}$
    \STATE Update history: $H_{t+1} = (H_t, A_t, B_t, R_{t+1, A_t, B_t})$
    \STATE Construct the reward vector $\tilde{R}_{t+1}$ with an optimistic estimation algorithm $E$
    \STATE No-regret update
    \ENDFOR
  \end{algorithmic}
\end{algorithm}
The essential part of Algorithm~\ref{alg:ETN} is the construction of the imagined reward vector $\tilde{R}_{t+1}$.

To elucidate the strategic underpinnings of adversarial bandit games, we introduce a comprehensive regret decomposition. This analytical framework sheds light on the subtleties of strategic decision-making against adversaries.
\begin{proposition}[Regret Decomposition]
\label{prop:regret-decomposition}
  Given any action $a \in \mathcal{A}$, we can dissect the one-step regret as follows:
  \begin{align*}
     \mathbb{E}\left[ R_{t+1, a, B_t} - R_{t+1, A_t, B_t} | \theta \right] = (I) + (II) + (III),
  \end{align*}
  where:
  \begin{align*}
  (I) &= \mathbb{E}\left[ \tilde{R}_{t+1}(a) - \tilde{R}_{t+1}(A_t) | \theta \right], \\
  (II) &= \mathbb{E}\left[ f_{\theta}(a, B_t) - \tilde{R}_{t+1}(a) | \theta \right], \\
  (III) &= \mathbb{E}\left[ \tilde{R}_{t+1}(A_t) - f_{\theta}(A_t, B_t) | \theta \right].
  \end{align*}
\end{proposition}
The aggregation of term $(I)$ essentially quantifies the adversarial regret $\mathcal{R}_{\text{full}}(T, \operatorname{adv}, \tilde{R})$ within a bounded sequence $\tilde{R}_{t+1}$. Through the application of sufficient optimism, term $(II)$ and $(III)$ enables the realization of $\sqrt{T}$-type regret minimization. Thus, the OTN framework facilitates the attainment of sublinear regret, contingent upon the integration of well-designed optimism and no-regret algorithms. We now proceed to examine various algorithms that seamlessly integrate within the \cref{alg:ETN} framework.

The \emph{importance-weighted estimator} (IWE)~\citep{lattimore2020bandit} stands as a cornerstone in the realm of bandit algorithms. When amalgamated with Hedge, it forms the basis of the renowned Exp3 algorithm~\citep{auer2002nonstochastic,kocak2014efficient,neu2015explore}, hereafter referred to as IWE-Hedge. Furthermore, the integration of IWE with Regret Matching (RM) introduces the IWE-RM strategy for bandit games. The nuances of both IWE-Hedge and IWE-RM, alongside their implications for adversarial regret, are meticulously outlined in~\cref{sec:iwe}.

\begin{proposition}[IWE-Hedge (Exp3) Analysis]
\label{prop:IWE-Hedge}
  Engaging in a bandit game utilizing the IWE-Hedge strategy results in 
    $\Re^*(T, \text{IWE-Hedge}) = \mathcal{O}(\sqrt{T|\mathcal{A}|\log|\mathcal{A}|}).$
\end{proposition}

The adaptation of Regret Matching (RM)~\citep{hart2000simple} through the lens of IWE under bandit feedback culminates in the innovative IWE-RM algorithm.
\begin{theorem}[Analysis of IWE-RM]
\label{thm:iwe-rm}
  Implementing the IWE-RM approach in bandit gameplay yields $\Re^*(T, \text{IWE-RM}) = \mathcal{O}(T^{2/3} |\mathcal{A}|^{2/3}).$
\end{theorem}

\begin{remark}
    The derivation of IWE-RM represents a novel contribution, previously unexplored in the literature. Detailed justification is available in Section~\ref{sec:iwe}.
\end{remark}

Crucially, while IWE hinges on bandit reward feedback, it does not inherently leverage opponent action information. By incorporating opponent action insights and reward structure knowledge $f_{\theta}$, we can refine our estimation strategies. Specifically, $f_{\theta}(a, b)$ for all pairs $(a, b) \in \mathcal{A} \times \mathcal{B}$ is approximated by $\tilde{f}_{t+1}(a, b)$, facilitating the construction of an imagined reward $\tilde{R}_{t+1}(a)$. Adopting a Gaussian distribution as the prior enhances the precision of mean $\mu_t(a,b)$ and variance $\sigma_t^2(a,b)$ estimations over time. The articulation of these estimation processes, alongside strategies for balancing exploration and exploitation, is detailed in \cref{app:update}.

In the stochastic bandit landscape, the Upper Confidence Bound (UCB)~\citep{auer2002finite} and Thompson Sampling (TS)~\citep{thompson1933likelihood} emerge as pivotal methods for instilling optimism in algorithmic choices. These methods are adeptly tailored to bandit games, with specific constructions for UCB and TS elucidated below, employing a parameter $\beta_t=\sqrt{2\log{|\mathcal{A}|\sqrt{t}}}$:
\begin{align*}
\begin{cases}
  \text{UCB:} & \tilde{f}_{t+1}(a, B_t) | H_{t+1} = \mu_{t}(a, B_t) + \beta_t \sigma_{t}(a, B_t),\\
               & \tilde{R}_{t+1}(a) = \min(\tilde{f}_{t+1}(a, B_t), 1), \forall a \in \mathcal{A}. \\
  \text{TS:}   & \tilde{f}_{t+1}(a, B_t) | H_{t+1} \sim \mathcal{N}(\mu_{t}(a, B_t), \sigma_{t}^2(a, B_t)),\\
               & \tilde{R}_{t+1}(a) = \min(\tilde{f}_{t+1}(a, B_t), 1), \forall a \in \mathcal{A}.
\end{cases}
\end{align*}
The synergy of UCB with Hedge, and its integration within the OTN framework, underscores a pathway to sublinear regret, as substantiated through our analytical endeavors (details in \cref{sec:analysis}). The comparative analysis of UCB-Hedge and UCB-RM, delineated below, highlights the efficacy of these approaches:
\begin{theorem}[Efficacy of UCB-Hedge and UCB-RM]
\label{thm:ucb_banditgame}
  Application of UCB-Hedge or UCB-RM in the context of bandit feedback games ensures 
    $\Re^*(T, \text{UCB-Hedge}) = \mathcal{O}(\sqrt{T \log |\mathcal{A}|} + \sqrt{\gamma_T\beta T}),$
    and similarly,
    $\Re^*(T, \text{UCB-RM}) = \mathcal{O}(\sqrt{T|\mathcal{A}|} + \sqrt{\gamma_T\beta T}).$
\end{theorem}

\begin{remark}
    The analytical framework for UCB-Hedge adopts a Bayesian perspective, contrasting with the frequentist approach taken in existing studies~\cite{sessa2019no}. This methodological divergence yields a $\sqrt{\gamma_T}$ improvement in our regret bounds.
\end{remark}
\subsection{Challenges with TS}\label{sec:divergence}
We explore the inefficacy of integrating Thompson Sampling (TS) with Regret Matching (RM) in certain bandit game contexts through a demonstrative counterexample.

\begin{example}[Best Response Player]\label{exp:matrix_game}
Consider matrix games characterized by a payoff matrix $\theta$:
  \begin{equation}\label{eq:matrix_game}
    \theta = \begin{bmatrix}
      1 & 1 - \Delta \\
      1 - \Delta & 1
    \end{bmatrix},
  \end{equation}
with $\Delta \in (0,1)$. In each round $t$, Alice selects her action $A_t \sim P_{X_t}$, while Bob, employing a best response strategy, chooses $B_t \sim P_{Y_t}$, aiming to maximize his payoff against Alice's choice. The reward for Alice, $R_{t+1, A_t, B_t} = \theta_{A_t, B_t}$, is determined without noise.
\end{example}

\begin{figure}[htbp]
  \centering
  \includegraphics[width=.95\linewidth]{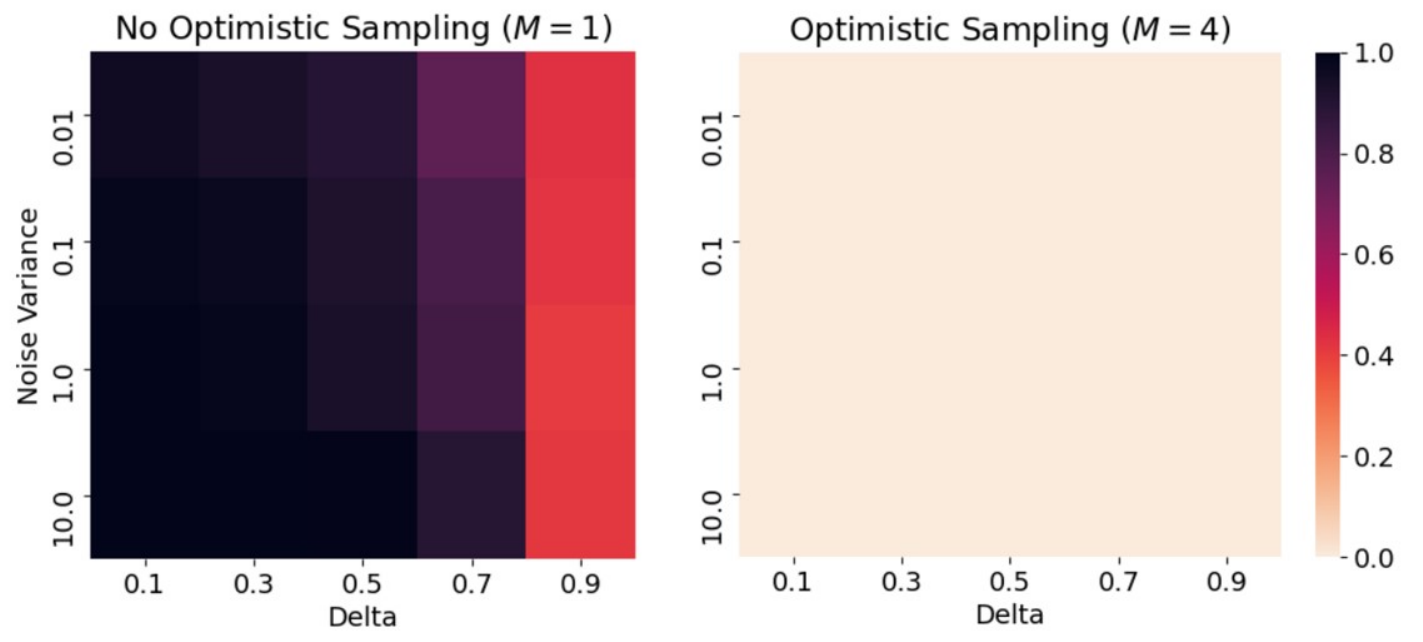}
  \caption{Illustration of failure probabilities for TS-RM versus OTS-RM strategies.} 
  \label{fig:counter}
\end{figure}

This setup reveals Bob's ability to exploit Alice if she adheres to a static strategy, leading to a regret increment proportional to $\Delta$. Define $\Omega_t$ as the scenario where Alice consistently chooses the second row, and Bob the first column, across all iterations. Given Alice's strategy begins uniformly, $\Re(T) \geq 2 \prob(\Omega_T) \Delta \cdot T$, suggesting potential for linear regret if $\Omega_t$ persists with consistent probability.

\begin{proposition}[Limitations of TS-RM]\label{prop:tsrm}
When Alice employs a uniform strategy and integrates TS with RM, for any $\Delta \in (0, 1)$ and noise variance $\sigma_n > 0$, there exists a constant $c(\Delta, \sigma_n) > 0$ ensuring $\prob(\Omega_t) \geq c(\Delta, \sigma_n)$ for all $t \geq 1$, highlighting sustained linear regret.
\end{proposition}

\subsection{Mitigating Strategy: Optimistic Sampling}\label{sec:optimistic_sampling}
Optimistic sampling, leveraging multiple posterior samples, emerges as a strategic remedy to enhance the estimator's optimism. Upon observing $H_{t+1} = (H_t, A_t, B_t, R_{t+1, A_t, B_t})$, optimistic sampling involves generating $M_{t+1}$ independent normal samples to construct an optimistic estimator for each action $a \in \mathcal{A}$:
\begin{align*}
    \tilde{f}^{\operatorname{TS}, j}_{t+1}(a, B_t) &\sim \mathcal{N}(\mu_{t}(a, B_t), \sigma_{t}(a, B_t)), \forall j \in [M_{t+1}], \\
    \tilde{f}^{\operatorname{OTS}}_{t+1}(a, B_t) &= \max_{j \in [M_{t+1}]} \tilde{f}^{\operatorname{TS}, j}_{t+1}(a, B_t), \\
    \tilde{R}^{\operatorname{OTS}}_{t+1}(a) &= \text{clip}_{[0, 1]}\left( \tilde{f}_{t+1}^{\operatorname{OTS}}(a, B_t) \right).
\end{align*}

Reflecting on \cref{exp:matrix_game}, implementing OTS alters the probability dynamics favorably over time, suggesting a reduction in the likelihood of $\Omega_t$. This adjustment confirms OTS's efficacy in countering the limitations of TS in adversarial settings.

\begin{theorem}[Advantages of OTS]\label{thm:ots_advantages}
Incorporating OTS alongside any full-information adversarial bandit strategy $\operatorname{adv}$ yields an enhanced regret bound:
  \begin{align*}
    \Re^*(T, \pi, \theta) \leq \Re_{\operatorname{full}}(T, \operatorname{adv}, \tilde{R}^{\operatorname{OTS}}) + \sqrt{\log(|\mathcal{A}|T) I(\theta; H_T) T},
  \end{align*}
underscoring the strategic benefit of integrating optimism into the estimation process.
\end{theorem}
\begin{remark}
The efficacy of the OTS-type method hinges on the bounded nature of the imagined rewards, $\tilde{R}^{\operatorname{OTS}} = (\tilde{R}^{\operatorname{OTS}}_{t+1}, t = 0, 1, \ldots)$, where each $\tilde{R}_t$ lies within the range $[0,1]^{\mathcal{A}}$. This constraint ensures that the reward estimations do not exceed plausible limits, thereby maintaining the integrity of the learning process. When the adversarial bandit algorithm $\operatorname{adv}$ is instantiated as Hedge, the regret bound under the OTS framework, $\Re_{\operatorname{full}}(T, \text{Hedge}, \tilde{R}^{\operatorname{OTS}})$, scales optimally as $\mathcal{O}(\sqrt{T  \log |\mathcal{A}|})$, demonstrating efficiency in a wide array of strategic scenarios. Similarly, if $\operatorname{adv}$ is implemented as RM, the regret bound, $\Re_{\operatorname{full}}(T, \text{RM}, \tilde{R}^{\operatorname{OTS}})$, achieves a comparable rate of $\mathcal{O}(\sqrt{T|\mathcal{A}|})$, signifying robustness across varying game dynamics. This adaptability is a testament to the versatility and practical utility of the OTS strategy in navigating the complexities of unknown game environments.
\end{remark}
\begin{remark}
While UCB and OTS both aim to augment algorithmic performance through optimism, OTS distinguishes itself by adopting stochastic bounds, offering practical advantages in complex environments where sampling efficiency and adaptability surpass traditional UCB computations.
\end{remark}

\section{Regret analysis}
\label{sec:analysis}
Our analysis roadmap is outlined as follows. First, we recall Proposition~\ref{prop:regret-decomposition}, which provides a general framework for regret decomposition:
\begin{align*}
    \mathbb{E}[ R_{t+1, a, B_t} - R_{t+1, A_t, B_t} | \theta] = (I) + (II) + (III).
\end{align*}
The summation of term $(I)$ leads to the reduction of adversarial regret $\Re_{\text{full}}(T, \operatorname{adv}, \tilde{R})$ for a sequence of bounded imagined rewards $\tilde{R}_{t+1}$. For term $(II)$, we define sufficient optimism as a condition to ensure $\sqrt{T}$-type regret. Furthermore, Proposition~\ref{prop:regret-decomposition} extends to both IWE and UCB, resulting in the derivation of by-product regret bounds for IWE-RM and UCB-RM, as detailed in Table~\ref{tab:regret} and further elaborated in Appendix~\ref{sec:analysis_details}.

\begin{definition}[Sufficient optimism]
  \label{asmp:optimism}
  The constructed imagined reward sequence $\tilde{R}^{\text{est}} = (\tilde{R}_{t}, t \in \mathbb{Z}_{++})$ is deemed optimistic if for any action $a \in \mathcal{A}$, it holds that
  \[
  \mathbb{P}( f_{\theta}(a, B_t) \ge \tilde{R}_{t+1}(a) | \theta) \le \mathcal{O}(1/\sqrt{T}).
  \]
\end{definition}
OTS adheres to Definition~\ref{asmp:optimism} by appropriately selecting $M_{t+1}$, whereas TS does not. The UCB sequence also satisfies this optimism criterion. Term $(III)$ can be bounded by the one-step information gain $I(\theta; R_{t+1, A_t, B_t} | H_t)$, utilizing the differential entropy of Gaussian distributions.

\begin{theorem}
  \label{thm:general}
  For any full information adversarial bandit $\operatorname{adv}$, utilizing an imagined reward sequence $\tilde{R}^{\operatorname{est}} = (\tilde{R}_{t}, t \in \mathbb{Z}_{++})$ constructed by an estimation algorithm $\operatorname{est}$ that satisfies Definition~\ref{asmp:optimism}, the combined algorithm $\pi = \pi^{\operatorname{adv-est}}$ enjoys the regret bound
  \begin{align*}
    \Re^*(T, \pi, \theta) \leq \Re_{\operatorname{full}}(T, \operatorname{adv})+\sqrt{ \beta I(\theta ; H_T)  T},
  \end{align*}
where $\Re_{\operatorname{full}}$ is as defined in Definition~\ref{def:regret_full}, and $\beta = \mathcal{O}(\log |\mathcal{A}| T)$.
\end{theorem}
\emph{Proof sketch.} We start with a general regret decomposition for any imagined reward sequence, bounding each term separately. The first term is bounded by reducing it to adversarial regret. For the second term, we utilize generic UCB and LCB sequences to derive a general regret bound. The third term is bounded using an information-theoretic quantity. Combining these bounds yields the overall regret upper bound, with detailed proofs available in Appendix~\ref{sec:analysis_details}.

To quantify the uncertainty reduction in specific game structures upon observing new information, we introduce the maximum information gain:
\begin{definition}[Maximum Information Gain]
  \label{def:max_infor_gain}
  \begin{align*}
    \gamma_T := \max_{A_{0:T}, B_{0:T}} I(\theta; A_{0}, B_0, \ldots, A_{T-1}, B_{T-1}),
  \end{align*}
where $I(X; Y)$ denotes the mutual information between random variables $X$ and $Y$.
\end{definition}
\begin{remark}
  A notable property of Gaussian distributions is that the information gain does not depend on the observed rewards, as the posterior covariance of a multivariate Gaussian is determined solely by the sampled points. Consequently, the maximum information ratio $\gamma_T$ in Definition~\ref{def:max_infor_gain} is well-defined, implying that $I(\theta; H_T) = I(\theta; A_0, B_0, \ldots, A_{T-1}, B_{T-1}) \le \gamma_T$.
\end{remark}
The bounds on $\gamma_T$ for various commonly used covariance functions, including finite-dimensional linear, squared exponential, and Matern kernels, are derived from~\cite{srinivas2009gaussian} and detailed in Appendix~\ref{sec:analysis_details}.
\begin{remark}
    The integration of information on opponent action and reward structures emerges as a potent strategy to counteract the challenges inherent in multi-agent settings, often referred to as the curse of multi-player. Notably, when employing squared exponential kernels to model the reward dynamics, the regret for OTS-Hedge is refined to $\mathcal{O}\left(\left(\sqrt{\log |\mathcal{A}|} + \sqrt{\log(|\mathcal{A}|T) \log(T)^{d+1}}\right) \sqrt{T}\right)$, a stark contrast to the polynomial dependence on the action spaces $\mathcal{A} \times \mathcal{B}$ observed in traditional approaches. This deviation from conventional adversarial algorithms, exemplified by the EXP3's regret bound of $\mathcal{O}(\sqrt{T \mathcal{A} \log \mathcal{A}})$, underscores the inefficiency of models that neglect opponent actions. By leveraging such information, OTS-Hedge not only alleviates the curse of multi-player but also significantly surpasses the limitations of standard adversarial bandit methods, echoing the benefits of a full-information framework. This paradigm shift, driven by the strategic exploitation of opponent actions and the nuanced understanding of reward structures, marks a significant leap towards deciphering the complexities of multi-player environments and setting a new benchmark for adversarial strategies.
\end{remark}

\section{Numerical studies and applications}
\label{sec:experiments}

We conduct evaluations of the proposed algorithms on random matrix games and two real-world applications. IWE-Hedge (the classical Exp3 algorithm) and IWE-RM are used as baseline comparisons, both adapted for bandit feedback scenarios. Additionally, Hedge and RM, which necessitate full information, are also employed. Within the OTN learning framework (Algorithm~\ref{alg:ETN}), four algorithms: OTS-Hedge, OTS-RM, UCB-Hedge, and UCB-RM, are compared against these baselines. It is noteworthy that GP-MW~\cite{sessa2019no} is referred to as UCB-Hedge in this context. The performance metric used is the average expected regret, with a detailed definition of performance metrics and algorithm settings provided in Appendix~\ref{appexp:metric}.

\subsection{Random Matrix Games}
\label{sec:rmg}
We assess various algorithms in repeated two-player zero-sum matrix games, where each entry of the payoff matrix is an independently and identically distributed random variable from the uniform distribution $[-1,1]$. Players have $M$ actions each, forming a square payoff matrix of size $M$. The experiment is run for a total of $T=10^7$ rounds, during which players receive noisy rewards $\pm \Tilde{r}_t$, with $\Tilde{r}_t=A_{ij}+\epsilon_t$ and $\epsilon_t\sim\mathcal{N}(0, 0.1)$. The analysis covers different matrix sizes ($M=10,50,70,100$), with each scenario tested across $100$ independent simulation runs. Performance is averaged over these runs against varying opponent models, detailed further in Appendix~\ref{appexp:opponent}.
\begin{figure*}[!htbp]
     \centering
     \begin{subfigure}[b]{.3\linewidth} 
         \centering
         \includegraphics[width=0.85\linewidth]{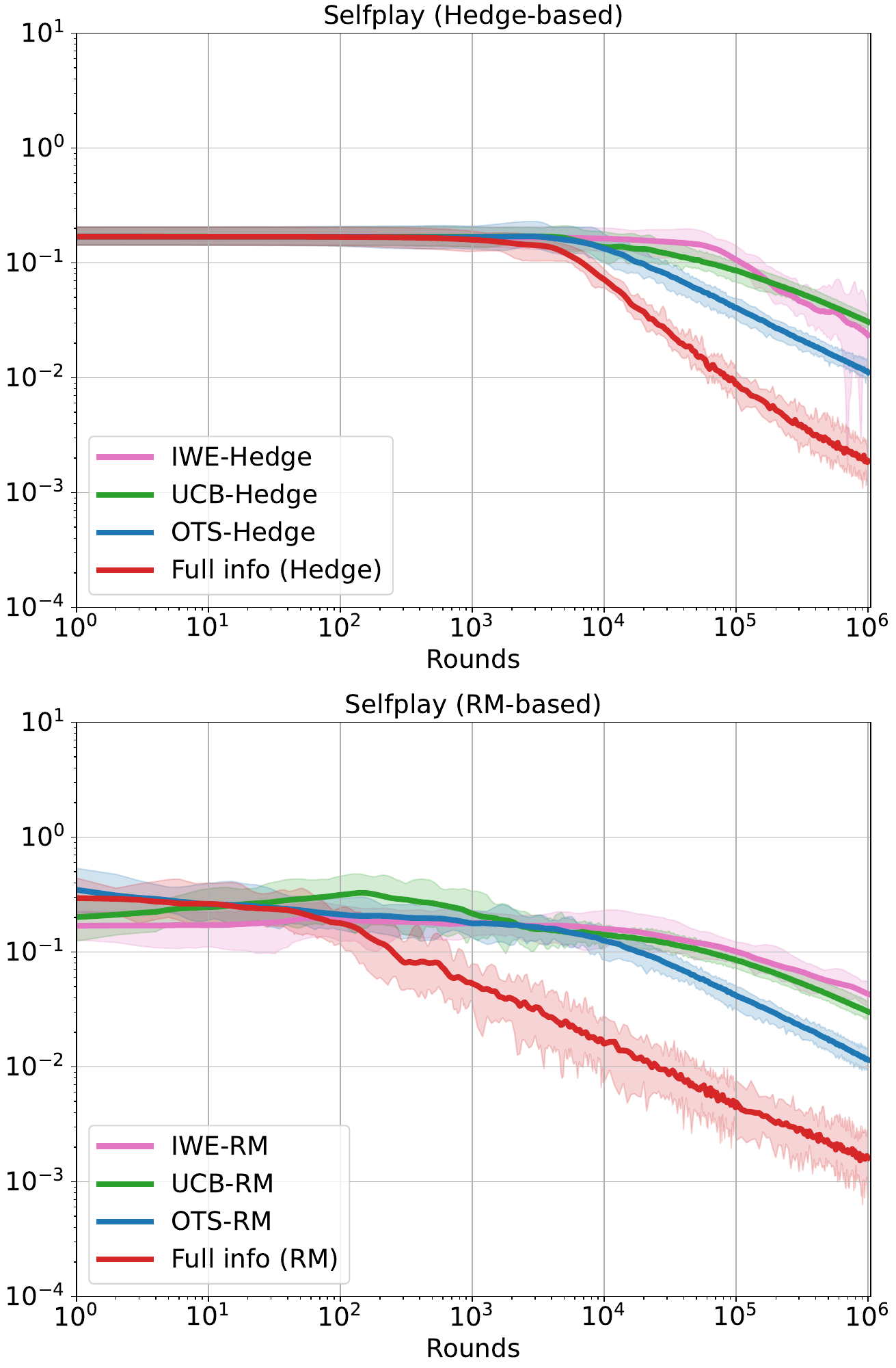}
         \caption{Self-play}
         \label{subfig: 10a}
     \end{subfigure} 
     \hspace{-6mm}
     \begin{subfigure}[b]{0.3\linewidth} 
         \centering
         \includegraphics[width=0.85\linewidth]{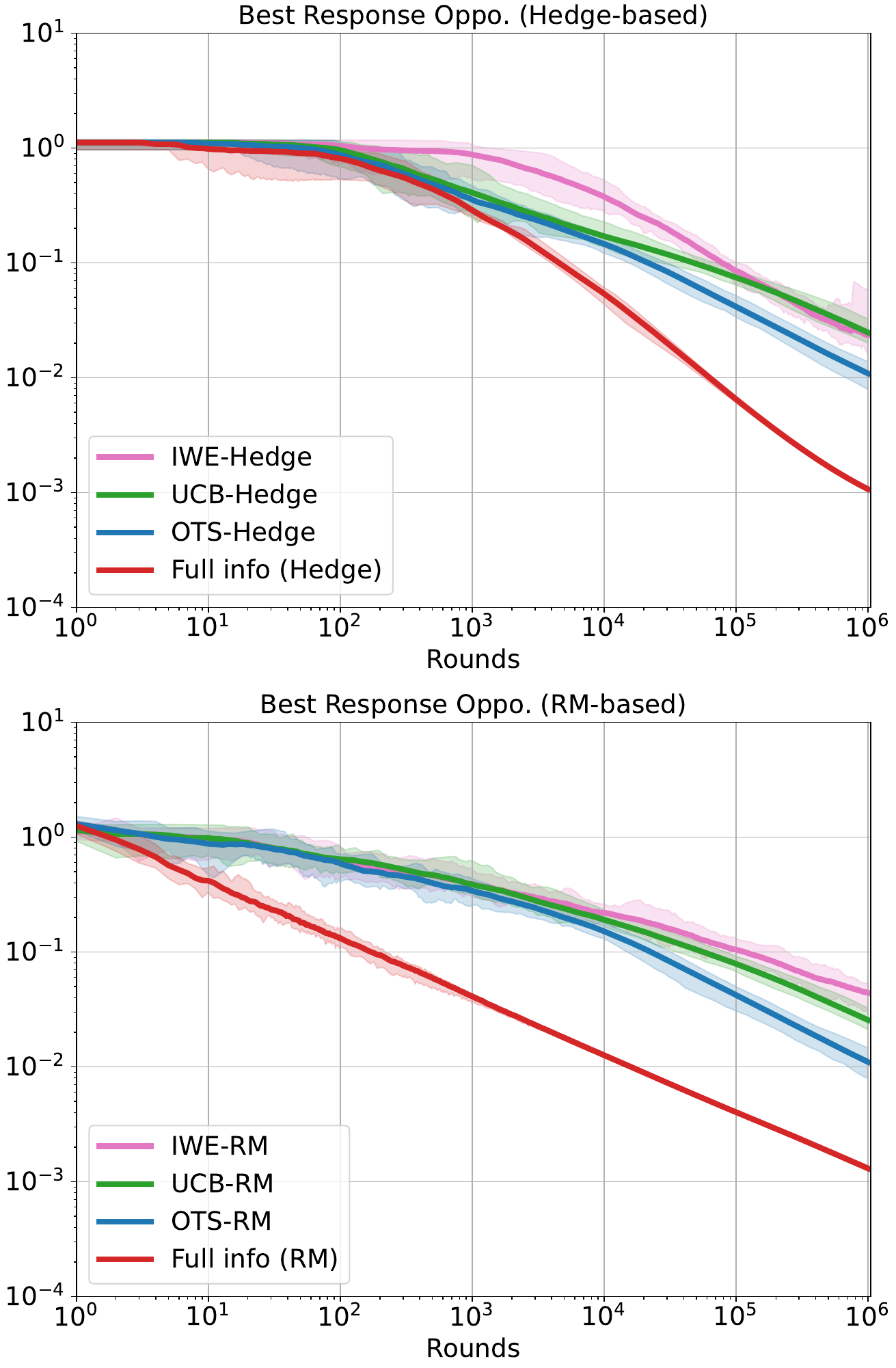}
         \caption{Best-response opponent}
         \label{subfig: 10b}
     \end{subfigure}
     \hspace{-6mm}
     \begin{subfigure}[b]{0.3\linewidth} 
         \centering
         \includegraphics[width=0.85\linewidth]{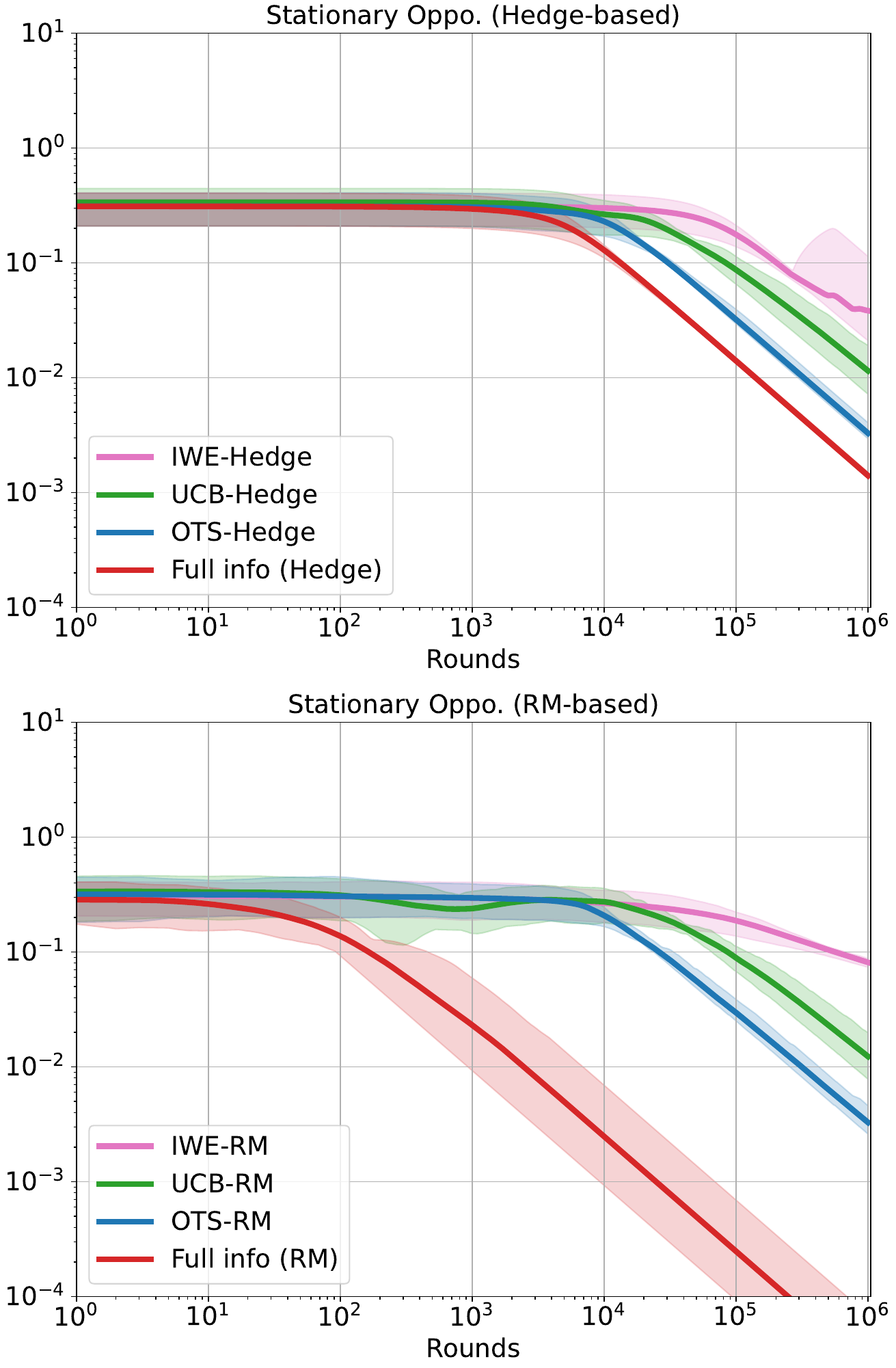} 
         \caption{Stationary opponent}
         \label{subfig: 10c}
     \end{subfigure}
        \caption{Averaged regrets for different opponents in a $50\times50$ matrix game.} 
        \label{fig:rdnmtx01} 
\end{figure*}

\subparagraph{Self-play.}
In the self-play scenario, where both players employ the same algorithm, \cref{subfig: 10a} reveals that algorithms leveraging the game's structure significantly outperform the IWE baselines, particularly in smaller matrix sizes. Notably, OTS-based algorithms demonstrate quicker convergence than their UCB-based counterparts, with RM algorithms showing an earlier reduction in average regret compared to Hedge algorithms.

\subparagraph{Best-response opponent.}
This scenario introduces a best-response opponent, fully informed about the matrix $A$ and the player's current strategy, hence always opting for the action that minimizes the player's expected payoff. \cref{subfig: 10b} shows that all algorithms within our proposed framework surpass the performance of the IWE baselines, with OTS-based algorithms again displaying superior convergence speeds relative to UCB-based methods.

\subparagraph{Stationary opponent.}
The focus here is on a stationary opponent whose strategy is a fixed probability distribution over the action space. The average regret in this context is indicative of an algorithm's effectiveness at exploiting the opponent's static behavior. As depicted in \cref{subfig: 10c}, OTS estimators confer significant advantages in exploiting such opponents over IWE-based estimators, showcasing the efficacy of leveraging game structure in algorithm design.

\subparagraph{Non-stationary opponent.}
Contrasting with previous settings, the opponent's strategy here varies non-statically, altering every $50$ rounds based on a predefined pattern. The game matrix $A \in \mathcal{R}^{10 \times 5}$, with entries drawn from $\mathcal{N}(0.5, 2.0)$, presents a challenging dynamic environment. The algorithms' robustness is evaluated over $1000$ rounds across $100$ simulation runs, with the rewards' distribution and performance metrics summarized in \cref{fig:rdnmtx02} and \cref{tab:rdnmtx01}, respectively. The OTS algorithms notably yield fewer negative rewards and higher average rewards than their IWE counterparts, illustrating enhanced resilience in face of strategic variability.

\begin{table}[htbp]
\centering
\begin{tabular}{ccc}
\hline
          & return$<0$ & mean return \\ \hline
IWE-Hedge & 19.4\% & 1.24        \\
IWE-RM    & 12.6\% & 1.50        \\
OTS-Hedge & 2.5\%  & 1.55        \\
OTS-RM    & 8.8\%  & 1.55      \\ \hline 
\end{tabular}
\caption{Returned rewards of IWE and OTS against a non-stationary opponent.} 
\label{tab:rdnmtx01}
\end{table}

\begin{figure*}[!htbp]
    \centering
    \includegraphics[width=0.8\textwidth]{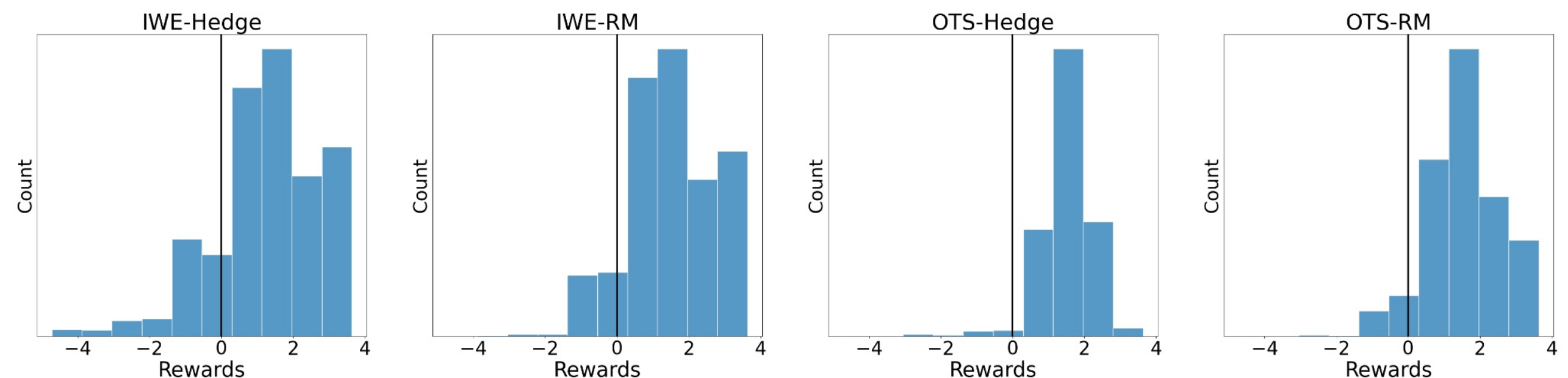} 
    \caption{Reward histograms of different algorithms against a non-stationary opponent, illustrating the performance disparity between algorithm types. OTS-based and UCB-based algorithms demonstrate significant efficiency, particularly in adapting to the jammer's changing strategies.}
    \label{fig:rdnmtx02} 
\end{figure*}

\subsection{Two-player radar signal processing: Linear game} \label{sec:radar}

The anti-jamming problem, a critical issue in signal processing literature~\citep{song2011mimo}, is modeled as a non-cooperative game between a radar and a jammer. The strategic interaction at the signal level involves both parties adjusting their transmitted signals' parameters to achieve opposing objectives: the radar aims to avoid signal interference by differing its carrier frequency from the jammer's, while the jammer attempts to match it.
This competition is formalized in the frequency domain as a linear game, with the signal-to-interference-plus-noise ratio (SINR) serving as the reward function $\text{SINR}(a,b;\theta)$. We conduct experiments comparing the average regret of various algorithms against an adaptive jammer, which updates its action based on the radar's most recent $10$ actions. As depicted in Fig.\ref{fig:radar_jammer}, OTS-based and UCB-based algorithms significantly outperform IWE-based algorithms, underscoring the advantages of adaptive strategies in anti-jamming scenarios.
\begin{figure}[!htbp]
     \centering
     \begin{subfigure}[b]{.49\linewidth} 
         \centering
         \includegraphics[width=\linewidth]{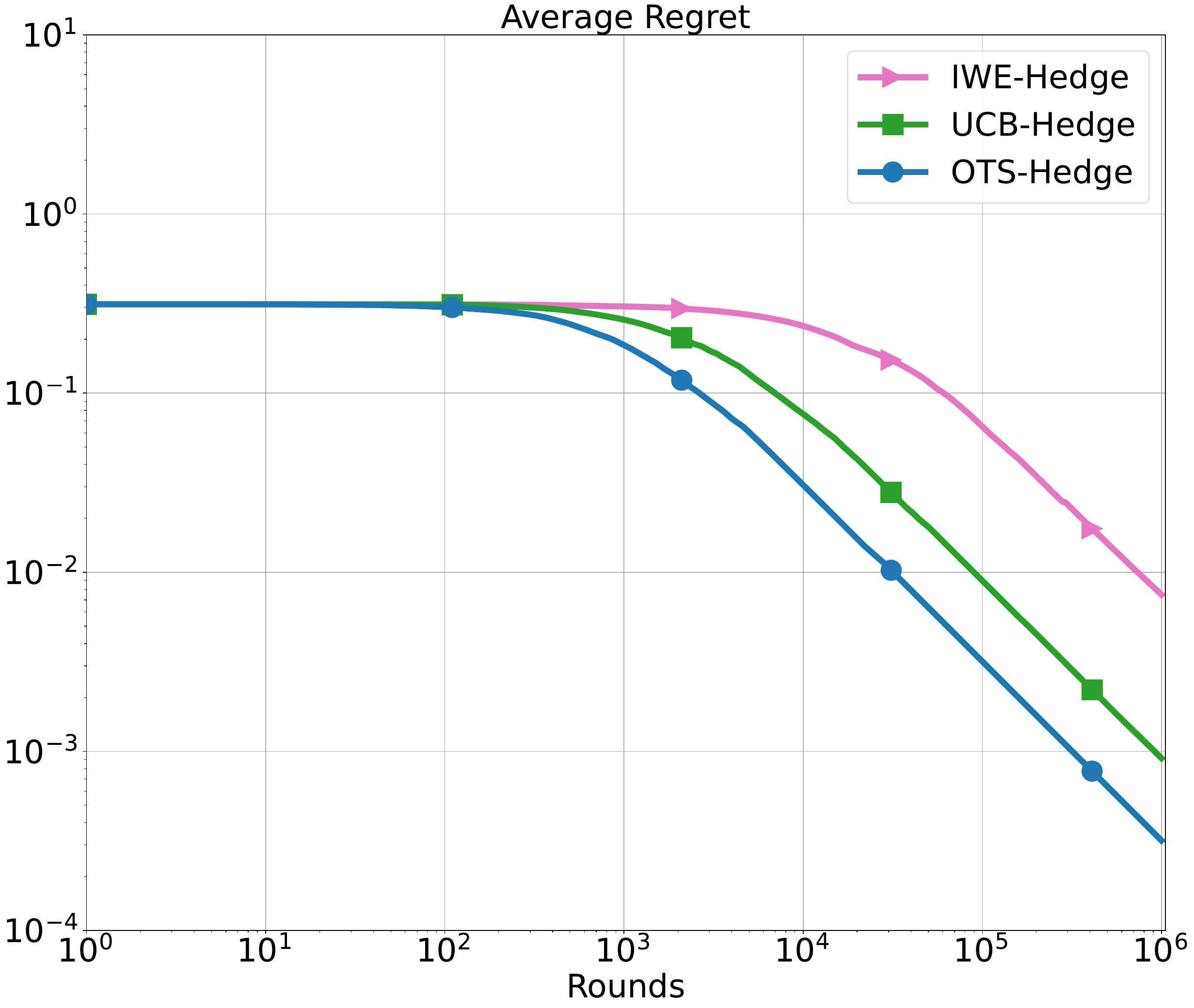}
         \caption{Hedge-based}
     \end{subfigure} 
     \begin{subfigure}[b]{0.49\linewidth} 
         \centering
         \includegraphics[width=\linewidth]{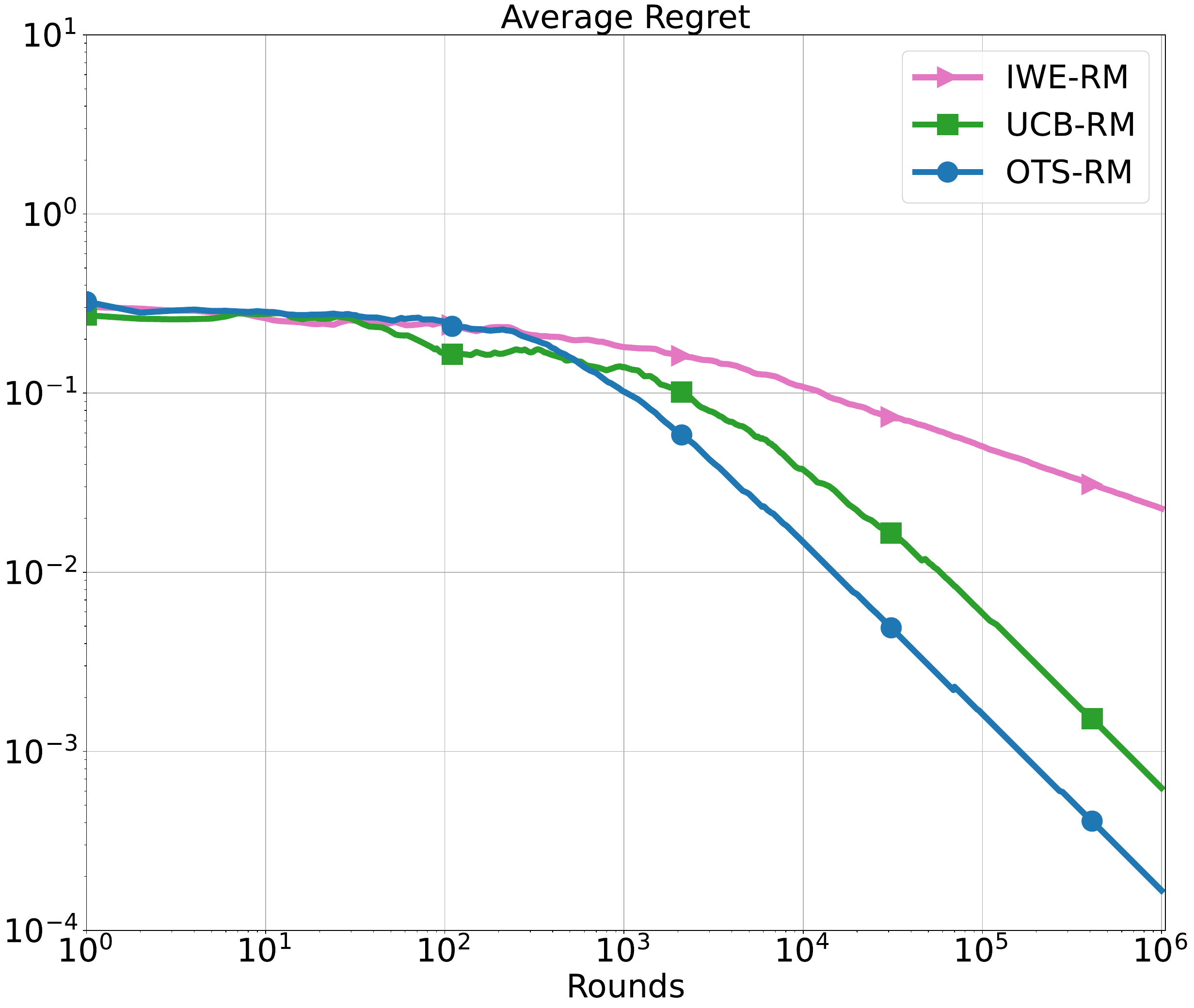}
         \caption{RM-based}
     \end{subfigure}
        \caption{Averaged regrets in the anti-jamming problem demonstrate the superior efficiency of OTS-based and UCB-based algorithms over IWE-based ones. Notably, OTS-Hedge requires significantly fewer samples to achieve a comparable level of average regret, illustrating its robustness and adaptability. Specifically, to achieve an average regret of $10^{-2}$, OTS-Hedge requires only $3.7\%$ and $31\%$ of the samples needed by IWE-Hedge and UCB-Hedge, respectively; OTS-RM uses only 0.1\% of IWE-RM's, 25\% of UCB-RM's, and 53\% of OTS-Hedge's samples.}
        \label{fig:radar_jammer} 
\end{figure}

Further details on the anti-jamming game setting, including experimental setup and algorithm parameters, are provided in \Cref{appexp:radar}, offering readers a comprehensive understanding of the methodologies employed to achieve these results.

\subsection{Repeated Traffic Routing: Kernelized Game}\label{sec:repeated}
The traffic routing problem, as derived from transportation studies, is formulated as a multi-agent game on a directed graph. In this model, each node pair signifies a distinct player, tasked with routing $U^i$ units from an origin to a destination node. The objective is to minimize travel time, influenced by the cumulative occupancy of traversed edges, making it the reward function. The set of actions available to each player comprises all feasible routes within the graph, with rewards inversely proportional to travel times. For an in-depth exploration of the setup, refer to \cref{app:traffic}.

Utilizing the Sioux-Falls road network dataset \cite{bar2015transportation}, we simulate a network of $24$ nodes and $76$ edges, engaging $N=528$ players. The action space $\mathcal{A}^i$ for player $i$ includes up to the $5$ shortest paths, excluding any route exceeding triple the shortest path's length. A Gaussian process models the reward function's correlations, as detailed in \cref{app:traffic}, following the methodology of GP-MW \cite{sessa2019no}, here referred to as UCB-Hedge.
Performance metrics extend beyond regret to encompass average congestion, offering a holistic view of traffic dynamics \cite{sessa2019no}. As depicted in \cref{fig:traffic}, OTS and UCB algorithms surpass IWE counterparts, highlighting the efficacy of the newly introduced OTS-Hedge, OTS-RM, and UCB-RM algorithms in surpassing UCB-Hedge's performance.

\begin{figure}[!htbp]
    \centering
    \includegraphics[width=\linewidth]{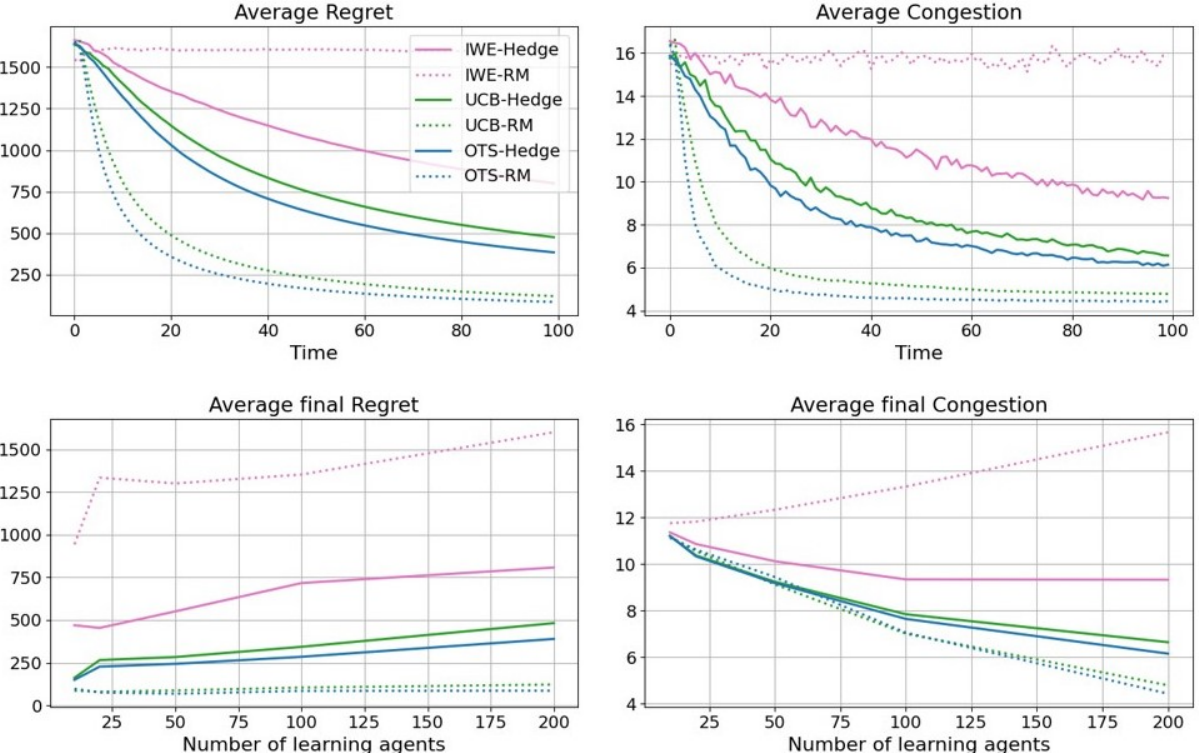} 
    \caption{Illustrating average regret and congestion for the traffic routing challenge. Notably, OTS-RM showcases unparalleled efficiency, requiring significantly fewer samples than its counterparts to achieve benchmark regret and congestion levels, thereby underscoring its capacity to manage hundreds of agents with reduced sample complexity. {Specifically, to achieve an average regret of $750$, OTS-RM requires $8\%$ of IWE-Hedge's, $16.3\%$ of UCB-Hedge's, $21.6\%$ of OTS-Hedge's, and $72.7\%$ of UCB-RM's samples. For an average congestion of $10$, OTS-RM uses $5.6\%$ of IWE-Hedge's, $15.3\%$ of UCB-Hedge's, $20\%$ of OTS-Hedge's, and $66.7\%$ of UCB-RM's samples. Moreover, OTS-RM and can effective handle hundreds of decision-makers in the most sample-efficient manner.}}
    \label{fig:traffic}
\end{figure}

\section{Conclusions}
\label{sec:conlcu}
This study presents a breakthrough in leveraging opponents' actions and reward structures through Thompson Sampling (TS)-inspired algorithms, which markedly optimize experimental resources, reducing costs by more than tenfold relative to traditional methods. We also introduce the Optimism-then-NoRegret (OTN) learning framework, an adaptive strategy for mastering the intricacies of unknown games, which subsumes various algorithmic approaches as special cases. Our proposed techniques have demonstrated significant performance improvements in both simulated and real-world settings. Future directions point towards integrating emerging TS approximation methods within complex models\citep{li2022hyperdqn,li2023efficient,li2024hyperagent}, like deep neural networks, with the OTN framework and the Optimistic Thompson Sampling (OTS) algorithm, especially in more complex multi-agent environments. This convergence is expected to refine decision-making processes, bolster model robustness, and widen the scope of our methodologies to encompass a greater array of machine learning and artificial intelligence challenges. The anticipated fusion of advanced sampling techniques with the OTN framework heralds a new frontier in machine learning research, brimming with promising prospects for innovation and application.

\section*{Impact statement}
Our research presents advancements in Machine Learning through the development of Thompson sampling-type algorithms for multi-agent environments with partial observations, applicable in fields like traffic routing and radar sensing. By significantly reducing experimental budgets and demonstrating a logarithmic dependence of the regret bound on the action space size, our work contributes to more efficient and equitable decision-making processes.

Ethically, this advancement promotes the responsible use of technology by ensuring broader access and reducing resource consumption. Societally, the potential applications of our research promise to enhance the functionality and safety of critical infrastructure, aligning with the goal of advancing public welfare through technological innovation. We are committed to further exploring these ethical and societal implications, ensuring that the deployment of our methodologies actively considers potential impacts to avoid unintended consequences.

In aligning with the conference's guidelines, we recognize the established ethical impacts and societal implications common to advancements in Machine Learning. However, we have endeavored to provide a concise discussion reflective of our work's unique contributions and potential future effects.

\bibliography{ref}
\bibliographystyle{icml2024}

\newpage
\appendix
\onecolumn
\section*{\Large Appendix: Optimistic Thompson Sampling for No-Regret Learning in Unknown Games}

 {
  \tableofcontents
 }

\newpage

\section{Additional discussion on related works}
\label{sec:addtional-related}
The Nash regret defined in \cite{o2021matrix} is only meaningful when facing a best-response opponent.
\begin{definition}{Nash equilibrium and Nash value}
The Nash value is defind as 
    \begin{align*}
        V^* = \max_{P \in \mathcal{D}(\mathcal{A})} \min_{Q \in \mathcal{D}(\mathcal{B})} \E[A \sim P, B \sim Q]{ f_{\theta}(A, B) },
    \end{align*}
and corresponding optimum $(P^*, Q^*)$ are the Nash equilibrium.
\end{definition}
\begin{definition}{Nash regret in \citep{o2021matrix}}
The Nash regret in step $t$ is defined as
    \begin{align}
    \label{eq:def-nash-regret-t}
        \operatorname{NashRegret}_t = \E{ V^* - R_{t+1, A_t, B_t} }
    \end{align}
    and the total Nash regret is defined as
    \begin{align}
        \label{eq:def-total-nash-regret}
        \operatorname{NashRegret}(T) = \sum_{t= 1}^T \operatorname{NashRegret}_t
    \end{align}
\end{definition}
We can only guarantee that $\operatorname{NashRegret}_t$ defined in \cref{eq:def-nash-regret-t} is positive when the opponent is the best response player with full knowledge on the matrix, which is unrealistic.
As for general opponent, $\operatorname{NashRegret}_t$ can be negative, which lose the meaning of `regret'. This is because even we can derive sublinear $\operatorname{NashRegret}(T)$, we could infer anything about the intermediate behavior of the $A$-player.
A special case is that if the $A$-player can always exploit the weakness of $B$-player, the total Nash regret can be linearly decreasing to $- \infty$, which is obviously `sublinear'.
However, this exploiting situation of the two players should be distinctive to the case that two players are playing Nash equilibrium and resulting $0$ total Nash regret, which is also `sublinear'.

As for our definition of adversarial regret in \cref{eq:adv_regret} in this repeated game setting with unknown reward function, we expect to measure how the $A$-player could exploit the weakness of $B$-player as time going on. The negative regret would not appear in any cases of the opponent. This is our motivation to use adversarial regret. 

\section{Description of bandit games and full information game}\label{app:game}
In this work, we consider three representative game forms: the matrix game, linear game, and kernelized game, as summarized in Table~\ref{tab:game}. 
\begin{table}[htbp]
\centering
\caption{Examples of bandit games.}
\small 
\begin{tabular}{cccc}
\toprule
& \textbf{Matrix game} &  \textbf{Linear game}&  \textbf{Kernelized game}   \\ \midrule
\multirow{2}{*}{Mean Reward} & $f_{\theta}(a, b) = \theta_{a, b}$  & $f_{\theta}(a, b)=\phi(a,b)^\top\theta$ &   $f_{\theta}(a, b)=\theta(a, b) $\\
  & $\theta\in\R^{\actions \times \mathcal{B}}$ & $\phi(a,b)\in\R^d$, $\theta\in\R^d$ & $\theta(a, b)\in\R$ is mean of a Gaussian Process \\ 
\bottomrule
\end{tabular}
\label{tab:game}
\end{table}

\subsection{Various bandit games}
\begin{assumption}
\label{assump:noise-likelihood}
    The corruption noise $W_t=Y_{t+1, A_{t}, B_t}-f_{\theta}(A_t, B_t)$ is assumed to be zero-mean Gaussian noise and is independent at each time $t$.
\end{assumption}
\begin{example}[Matrix games]
\label{example:matrix_linear}
In a matrix game, the reward function simplifies to $f_{\theta}(a, b) = \theta_{a, b}$. In this degenerate setting, $\theta$ can be considered as the utility matrix for Alice. 
\end{example}

\begin{example}[Linear games]
\label{example:gaussian_linear}
In a linear game, a known feature mapping $\phi: \actions \times \mathcal{B} \mapsto \R^d$is defined, and the mean reward function is given by $f_{\theta}(a, b) = \phi(a, b)^\top \theta$, where the reward is linear in the feature. We assume that the random parameter $\theta$ follows the normal distribution $ {N}(\mu_p, \Sigma_p)$, and the reward noise $W_{t+1} = Y_{t+1, A_t, B_t} - f_{\theta}(A_t, B_t)$ is normally distributed with with mean zero and variance $\sigma_w^2$, independent of $(H_t, A_t, B_t, \theta)$.  
\end{example}
In the \cref{sec:repeated}, the reward structure in the repeated traffic routing problem is modeled using a kernel function, which is referred to as Kernelized games. 
\begin{example}[Kernelized games]
  \label{example:kernel}
In kernelized games, we consider the case where the reward function $f_{\theta}$ is sample from a Gaussian process. The stochastic process $(f_{\theta}(a, b): (a, b) \in \actions \times \mathcal{B})$ follows a multivariate Gaussian distribution, where the mean function is denoted as $\mu(a, b) = \E{f_{\theta}(a, b)}$ and covariance (or kernel) function is denoted as $k((a, b), (a', b')) = \E{(f_{\theta} (a, b) - \mu(a, b) )(f_{\theta} (a', b') - \mu(a', b') )}$.
The kernel function $k((a,b), (a', b'))$ measures the similarity between different action pairs $(a,b), (a', b') \in \actions \times \mathcal{B} $ in the game.
We assume that the function $f_{\theta}$ is sampled from a  Gaussian process prior $GP(0, k((a, b), (a', b')))$, the reward noise $W_{t+1} = Y_{t+1, A_t, B_t} - f_{\theta}(A_t, B_t)$ is independent of $(H_t, \theta, A_t, B_t)$, and $(W_{t}: t \in \Z_{++})$ is an i.i.d sequence following $N(0, \sigma_w^2)$.
\end{example}

\subsection{Full information feedback}
To introduce the proposed Optimism-then-NoRegret learning framework, we first consider the full information feedback setting where Alice can observe the mean rewards $r_t(a) = f_{\theta}(a, B_t)$ for all actions $a \in \actions$.
In this case, the problem can be solved using full-information adversarial bandit algorithms such as Hedge~\citep{freund1997decision} and Regret Matching (RM)~\citep{hart2000simple} applied to the sequence of adversarial reward vectors $(r_{t})_{t\in [T]} \in [0,1]^{\actions \times T}$. The procedure is summarized in \cref{alg:NR}, where $P_X$ denotes probability simplex proportional to $X$, and function $g_t: \Delta^{\actions} \times [0, 1]^\actions \mapsto \Delta^{\actions}$ in round $t$ is specified as follows:
$$
\textrm{Hedge:}\ g_{t, a}(X_t, r_t)= X_{t, a} \exp ( \eta_t r_t(a) ),\quad\textrm{RM:}\ g_{t, a}(X_t, r_t) = \max\left(0, \sum_{s=0}^t r_t(a) - r_t(A_s) \right).$$

In the full information setting, where $r_t(a) = f_{\theta}(a, B_t)$, the adversarial regret defined in~\cref{eq:adv_regret} translates to the following full information adversarial regret: 
\begin{definition}[Regret with Full Information]
  \label{def:regret_full_appendix}
  The full information adversarial regret of algorithm $\operatorname{adv}$ for arbitrary reward sequence $(r_t)_t$ is defined as
\begin{align}
  \Re_{\operatorname{full}} ( T, {\operatorname{adv} }, (r_t)_t ) = \max_{a\in \actions} \E{ \sum_{t= 0}^{T-1} r_t(a) -  r_t(A_t) }
\end{align}
\end{definition}
The following proposition provides the regret bounds for Hedge and RM algorithms in the full information feedback setting~\citep{freund1997decision}: 
\begin{proposition}[Regrets of Hedge and RM]
Consider playing in a full information feedback game with Hedge or RM algorithms. The regrets are bounded as follows: 
$$
\Re_{\operatorname{full}} ( T, {\operatorname{Hedge} }, (r_t)_t )=\mathcal{O}( \sqrt{T \log \actions} ),\ \Re_{\operatorname{full}} ( T, {\operatorname{RM} }, (r_t)_t )=\mathcal{O}( \sqrt{T\actions} ). 
$$
\end{proposition}
\begin{algorithm}[H]
    \caption{No Regret Update for Full-Information Feedback}
    \label{alg:NR}
    \begin{algorithmic}[1]
    \STATE{Initialize $X_1$}
    \FOR{round $t = 0, 1, \ldots, T - 1$ }
    \STATE Sample action $ A_t \sim P_{X_t} $,
    \STATE Observe full-information feedback $ f_{\theta}(a, B_t) $ for all $a\in\actions$,
    \STATE Update: $X_{t+1} = g_t(X_t, (f_{\theta}(a, B_t) )_{a\in \actions} )$.
    \ENDFOR
  \end{algorithmic}
\end{algorithm}

\section{Updating rules in various games}\label{app:update}
The \cref{assump:noise-likelihood} together with the stucture of mean reward functions gives the following Bayesian update rule of posterior. It is also possible to use Bayesian updated algorithm in frequentist setting, with a slightly different treatment.
\paragraph{Posterior distribution for Linear Gaussian model (Parametric).}
Let's consider the linear Gaussian model with a Gaussian prior $N(\mu_p, \Sigma_p)$ and noise likelihood $N(0, \sigma_w^2)$. Here are the key aspects of the model:
  \begin{itemize}
    \item Prior Assumptions: We assume a zero-mean prior with covariance $\Sigma_p = \sigma_p I$, where $\sigma_p$ is a scalar parameter satisfying $\sigma_p \leq 1$.
    \item Feature Map Assumptions: We assume that the feature map $\phi(a, b)$ satisfies $|\phi(a, b)| \leq 1$.
    \item Covariance Matrix Update: Given the initial covariance matrix $\Sigma_0 = \Sigma_p$, the covariance matrix at time $t+1$ is updated as: 
    \begin{align*}
      \Sigma_{t+1} =
      \left(\Sigma_{t}^{-1}+\frac{1}{\sigma_w^{2}} \phi(A_{t}, B_t) \phi(A_{t}, B_t)^{\top}\right)^{-1}
    \end{align*}
    \item Mean Vector Update: Given the initial mean vector $\mu_0 = \mu_p$, the mean vector at time $t+1$ is updated as:
    \begin{align*}
      \mu_{t+1}=\Sigma_{t+1}\left(\Sigma_{t}^{-1} \mu_{t}+\frac{R_{t+1, A_{t}, B_t}}{\sigma_w^{2}} \phi( A_{t}, B_t ) \right)
    \end{align*}
    \item $\sigma_t(a, b) = \norm{ \phi(a, b) }_{\Sigma_{t}} $ and $\mu_t(a, b) = \phi(a, b)^\top \mu_t$
    \item In the case where the feature $\phi(a, b) = e_{a, b}$ is a one-hot vector, we denote $n_t(a, b)$ as the counts of occurrences of $(a, b)$ up to time $t$, the posterior variance is given by:
    \begin{align*}
      \sigma_t(a, b) = \sqrt{ \frac{\sigma_w^2}{\sigma_w^2 / \sigma^2_p(a, b) + n_{t}(a, b)} }
    \end{align*}
  \end{itemize}

\paragraph{Posterior distribution for Gaussian Process (Non-paramatric).}
In the non-parametric case of a Gaussian process (GP), the posterior distribution remains Gaussian as well. Here are the relevant details:
  \begin{itemize}
    \item Notation: We define the vector $\mathbf{k}_t((a, b))$ and $\mathbf{R}_t$, and the matrix $\mathbf{K}_{t}$ as follows:
    \begin{align*}
      \mathbf{k}_t(a, b) & = [ k( (A_0, B_0) , (a, b) ), \ldots, k( (A_{t-1}, B_{t-1}) , (a, b) ) ]^\top \\
      \mathbf{R}_{t} & = [ R_{1, A_0, B_0}, \ldots, R_{t, A_{t-1}, B_{t-1}} ]^\top \\
      \mathbf{K}_t(i,j) & = k( (A_i, B_i), (A_{j}, B_{j} ) )
    \end{align*}
    \item Variance Assumption: We assume that the variance satisfies $k(x, x) \le 1$ for all $x \in \mathcal{X}$.
    \item Posterior Variance: The posterior variance at time $t$ is given by:
    $
      \sigma^2_t(a, b) = k((a, b), (a, b)) - \mathbf{k}_t((a, b))^\top (\mathbf{K}_t + \sigma^2 \mI_t) \mathbf{k}_t (a, b)
    $
    \item Posterior Mean: The posterior mean at time $t$ is given by:
    $
      \mu_{t} (a, b) = \mathbf{k}_t((a, b))^\top ( \mathbf{K}_t + \sigma^2 \mI_t )^{-1} \mathbf{R}_t
    $
    \item Relationship to Linear Gaussian Model: If the kernel $k((a, b), (a, b)) = \phi(a, b)^\top \Sigma_p \phi(a, b)$ is composed of basis functions, the GP reduces to the linear Gaussian model with a prior covariance matrix $\Sigma_p$. This ensures coherence between the linear Gaussian model and the kernel model assumptions.
  \end{itemize}

\section{Details for importance weighted estimator in \cref{sec:algorithm}}
\label{sec:iwe}

\paragraph{Importance-weighted estimator.}
For any measurable function $h$ and probability distribution $(X_{a})_{a\in \actions}$ over a finite support $\actions$, we construct importance weighted estimator
\begin{align*}
  \tilde{h}(a) = \indict_{A=a} \frac{h(A)}{X_{a}}, \forall a \in \actions
\end{align*}
which is an unbiased estimator:
\begin{align*}
  \E{\tilde{h}(a)} = \E{ \indict_{A=a} }  h(a) / X_a = h(a).
\end{align*}

\paragraph{Exp3: Hedge with importance weighted estimator}
In this work, we sometimes call Exp3 as IWE-Hedge. The celebrated Exp3 algorithm construct an estimate of reward vector as
\begin{align*}
  \tilde{R}_{t+1}(a) = 1 - \frac{ \indict_{A_t = a} (1 - R_{t+1, A_t, B_t} ) }{ X_{t, a} } 
\end{align*}
We can observe that $\tilde{R}_{t+1}(a)$ is unbiased conditioned on history $H_t$. Given that $X_t$ is $H_t$-measurable and $A_t$ is conditionally independent with $W_{t+1}$ and $B_t$ given $H_t$, and using the fact $\indict_{A_t = a} R_{t+1, A_t, B_t} = \indict_{A_t = a} R_{t+1, a, B_t}$, we have:
\begin{align*}
  \E[t]{\tilde{R}_{t+1}(a)} = 
  1 - \E[t]{ \indict_{A_t = a} \frac{1 - R_{t+1, a, B_t} }{ X_{t, a}} } 
  = 1 - \E[t]{ \indict_{A_t = a} } \frac{1 - \E[t]{ f_{\theta}( a, B_t) } }{ X_{t, a}} = \E[t]{ f_{\theta}(a, B_t)}.
\end{align*}
Exp3\citep{auer2002nonstochastic} updates the strategy using $ X_{t+1, a} \propto X_{t, a} \exp(\eta_t \tilde{R}_{t+1}(a) )$.

\paragraph{Regret matching with importance weighted estimator.}
Using the importance-weighted estimator, we can obtain an unbiased estimator for the regret $\Re_t \in \R^\mathcal{A}$ at round $t$: 
\begin{align}
  \label{eq:rm_est}
  \tilde{\Re}_{t, a} = \frac{\indict_{A_t = a} R_{t+1, A_t, B_t} }{X_{t, a}} - R_{t+1, A_t, B_t} \frac{\hat{X}_{t,A_t}}{X_{t,A_t}}
\end{align}
and update the strategy as follows: Let the cumulative estimated reward be $\tilde{C}_{t, a} = \sum_{s=0}^t \tilde{\Re}_{s, a} $,
\begin{align}
  \label{eq:update}
  \hat{X}_{t+1, a} = \begin{cases}
    { \tilde{C}^+_{t, a} }/{ \sum_{a \in \actions} \tilde{C}_{t, a}^{+}  }, & \text{if } \sum_{a \in \actions} \tilde{C}_{t, a}^{+} > 0,\\
    \text{arbitrary vector on simplex, e.g. } 1/\actions, & \text{otherwise}
  \end{cases}
\end{align}
Here, the sampling distribution $X_{t, a}$ is mixed with uniform distribution 
\begin{align}
  \label{eq:mixture}
  X_{t, a} = ( 1 - \gamma_t) \hat{X}_{t, a} + \gamma_t ( 1 / \actions ), \forall a \in \actions
\end{align}
The detailed algorithm for Importance-weighted estimator regret matching (IWE-RM) is as follows:
\begin{algorithm}[H]
  \caption{Importance-weighted estimator with Regret Matching (IWE-RM)}
  \label{alg:iwe-rm}
  \begin{algorithmic}[1]
  \STATE {{\bfseries{Input}}: init $X_1 = \hat{X}_1$ as uniform probability vector over $\mathcal{A}$ and sequence $(\gamma_t)_{t\ge 0}$} 
  \FOR{round $t = 0, 1, \ldots, T-1$ }
  \STATE Sample action $ A_t \sim P_{X_t} $
  \STATE Observe noisy bandit feedback $ R_{t+1, A_t, B_t} $.
  \STATE Construct regret estimator $\tilde{\Re}_{t}$ with importance weighted estimation by \cref{eq:rm_est}.
  \STATE Update $\hat{X}_{t+1}$ and $X_{t+1}$ with \cref{eq:update,eq:mixture} and $\gamma_t$.
  \ENDFOR
  \end{algorithmic} 
  \end{algorithm}
  \begin{fact}
  Importance weighted estimator $\tilde{R}_{t+1}$ at round $t$ is $\sigma(H_t, A_t, R_{t+1, A_t, B_t})$-measurable.
\end{fact}
\begin{remark}
  For any $t\in \NN$, the imagined reward vector $\tilde{R}_{t+1}$ constructed by importance weighted estimator satisfies $\E{ \tilde{R}_{t+1}(a) \given H_t, \theta} = \E{ f_{\theta}(a, B_t) \given H_t, \theta} = \E{ g_{\theta}(e_a, Y_t) \given H_t, \theta }$ for any $a \in \actions$. 
  Therefore, we have $\E{ \operatorname{pess}_{t+1} \given \theta } = \E{ \operatorname{est}_{t+1} \given \theta } = 0$.
\end{remark}
\begin{lemma}
	\label{lem:sum-max-square}
	For all real $a$, define $a^+ = \max\{a,0\}$. For all $a, b$, it is the case that
	\begin{align*}
	\left( (a+b)^+ \right)^2 \leq (a^+)^2 + 2 (a^+)b + b^2
	\end{align*}
\end{lemma}
\begin{proof}
	\(
	(a+b)^+ \leq (a^+ +b)^+ \leq \abs{a^+ + b}.
	\)
\end{proof}

\begin{lemma}
For all vector $v \in \R^{\mathcal{A}}$, define $v^+ = (v_{a}^+ )_{a \in \mathcal{A}}$.
  Following~\cref{alg:iwe-rm}, we have the important observation
	\[
	\<\tilde{C}_{t-1}^+,\tilde{\Re}_{t}> \leq 0
	\]
	\label{lem:blackwell}
\end{lemma}
\begin{proof}
Suppose at round $t$, Alice choose $A_t \sim P_{X_t}$ and receive the feedback $R_{t+1, A_t, B_t}$. By algorithm~\ref{alg:iwe-rm} and \cref{eq:rm_est,eq:update,eq:mixture},
If $\sum_{a} \tilde{C}_{t-1,a}^+ \leq 0$, then obviously  $\tilde{C}^+_{t-1,a}=0$ for all action $a \in \mathcal{A}$. 
Then, the lemma trivially holds.

Otherwise, we have
\begin{align*}
    \<\tilde{C}_{t-1}^+, \tilde{\Re}_{t}>
    & = R_{t+1, A_t, B_t} \left(   \frac{\tilde{C}_{t-1, A_t}^+}{ X_{t, A_t} } - \frac{ \hat{X}_{t, A_t}}{{X}_{t, A_t}}\sum_{a}\tilde{C}_{t-1,a}^+ \right) \\
    & = R_{t+1, A_t, B_t} \left(   \frac{\tilde{C}_{t-1, A_t}^+}{ X_{t, A_t} } - \frac{ \tilde{C}_{t-1, A_t}^+ /\sum_{a} \tilde{C}_{t-1, a}^+ }{{X}_{t, A_t}}\sum_{a}\tilde{C}_{t-1,a}^+ \right) = 0
  \end{align*}
  
\end{proof}

\begin{lemma}
  Following algorithm~\ref{alg:iwe-rm}, we have an important inequality
    \[
        \sum_{a} \left({\tilde{C}_{T, a}^+}\right)^2 \leq  \sum_{t=1}^T \sum_a \left( \tilde{\Re}_{t, a} \right)^2
    \]
\label{lem:regret-bound}
\end{lemma}
\begin{proof}
	Since from the update rule,
	\begin{align*}
	{\tilde{C}_{T, a}} = {\tilde{C}_{T-1, a}} + \tilde{\Re}_{T, a},
	\end{align*}
	by Lemma \ref{lem:sum-max-square} and \ref{lem:blackwell}:
	\begin{align*}
	\sum_{a} \left({\tilde{C}_{T, a}^+}\right)^2
	&\leq 
	\sum_a \left( \left({\tilde{C}_{T-1, a}^+}\right)^2
	+ 2 {\tilde{C}_{T-1, a}^+} {\tilde{\Re}_{T, a}}
	+ \left(\tilde{\Re}_{T, a} \right)^2 \right)\\
	&\leq \sum_{a} \left({\tilde{C}_{T-1, a}^+}\right)^2 + \sum_{i}(\tilde{\Re}_{T, a})^2
	\end{align*}
	By telescoping series, we conclude the lemma.
\end{proof}

\begin{proof}[Proof of Theorem~\ref{thm:iwe-rm}]
Notice that the action $A_t$ selected by Alice and the action $B_t$ selected by Bob is independent conditioned on history $H_t$ and $\theta$.
According to the algorithm~\ref{alg:iwe-rm},
the conditional expectation of the estimated immediate regret is
\begin{align*}
\E{\tilde{\Re}_{t, a} \mid \theta, H_t}
= \E{ f_{\theta}(a, B_t) \mid \theta, H_t } - \sum_{a} \hat{X}_{t, a} \E{ f_{\theta}(a, B_t) \mid \theta, H_t }
\end{align*}
\paragraph{Step 1 (Bounding the bias of estimated regret.)}
Recall the definition of immediate regret at time $t$ conditioned on history is
\[
    \Re_t(a) := \E{R_{t+1, a, B_t} - R_{t+1, A_t, B_t} \mid \theta, H_t} = \E{ f_{\theta}(a, B_t) \mid \theta, H_t } - \sum_{a} {X}_{t, a} \E{ f_{\theta}(a, B_t) \mid \theta, H_t }
\]
In the following, we use short notation $\E[t]{\cdot} = \E{ \cdot \mid H_t, \theta }$.
For any $a \in \mathcal{A}$, the difference with the estimated regret under conditional expectation is
\begin{align*}
\Re_t(a) - \E[t]{\tilde{\Re}_{t, a}} 
& = \sum_{a} (\hat{X}_{t, a} - X_{t, a}) \sum_{b} Y_{t, b} f_{\theta}(a, b) \\
& = \sum_{a} ( \gamma_t \hat{X}_{t, a} - \gamma_t /\mathcal{A} ) \sum_{b} Y_{t, b} f_{\theta}(a, b) \\
& \leq \sum_{a} \abs{ \gamma_t \hat{X}_{t, a} - \gamma_t /\mathcal{A} } \\
& \le \sum_{a} \left( \abs{ \gamma_t \hat{X}_{t, a}} + \abs{ \gamma_t /\mathcal{A} } \right) = 2 \gamma_t
\end{align*}
\paragraph{ Step 2 (Bounding the potential.)}
For any $a \in \mathcal{A}$,
\begin{align*}
    \E[0]{ \tilde{C}_{T, a} }
    & \leq \E[0]{ \tilde{C}_{T, a}^+ }
    = \E[0]{ \sqrt{ (\tilde{C}_{T,a}^+)^2} }
    \leq \E[0]{ \sqrt{ \sum_a \left( \tilde{C}_{T, a}^+ \right)^2} },
\end{align*}
where the last inequality is due to the Jensen inequality.
By lemma~\ref{lem:regret-bound} and taking expectation,
\begin{align*}
    \E[0]{ \sqrt{ \sum_a \left(\tilde{C}^+_{T,i} \right)^2} }\leq \E[0]{ \sqrt{ \sum_{t=1}^T \sum_a \left( \tilde{\Re}_{t, a} \right)^2 } }
\end{align*}
The RHS of the above inequality can be bounded as
\begin{align*}
    \E[0]{\sum_{t=1}^T \sum_a \left( \tilde{\Re}_{t, a} \right)^2} 
    & = \E[0]{\sum_{t=1}^T \sum_a \left(\frac{\indict_{A_t = a} R_{t+1, A_t, B_t} }{X_{t, a}} - R_{t+1, A_t, B_t} \frac{\hat{X}_{t,A_t}}{X_{t,A_t}} \right)^2} \\
    & = \E[0]{ \sum_{t=1}^T \sum_{a} 
    R_{t+1, A_t, B_t}^2 \left(
    \left(\frac{\indict_{A_t = a}}{X_{t,A_t}} \right)^2 -  \frac{2 \hat{X}_{t, A_t} }{ X_{t, A_t}^2} \indict_{A_t=a}
    + \frac{\hat{X}_{t, A_t}^2}{X_{t, A_t}^2} \right)}  \\
    & = \E[0]{
    \sum_{t=1}^T R_{t+1, A_t, B_t}^2
    \left( \frac{1}{X_{t, A_t}^2}
    - \frac{2 \hat{X}_{t, A_t} }{ X_{t, A_t}^2}
     + \abs{\mathcal{A}} \frac{\hat{X}_{t, A_t}^2}{X_{t, A_t}^2}
    \right)
    } \\
    & = \E[0]{
    \sum_{t=1}^T \E[t]{R_{t+1, A_t, B_t}^2
    \left( \frac{1}{X_{t, A_t}^2}
    - \frac{2 \hat{X}_{t, A_t} }{ X_{t, A_t}^2}
     + \abs{\mathcal{A}} \frac{\hat{X}_{t, A_t}^2}{X_{t, A_t}^2}
    \right)
     } },
\end{align*}
where we have the following derivation by the fact 
\[
    \E[t]{R_{t, a, b}^2} := \E[t]{ (f_{\theta]}(a, b) + W_{t+1})^2 } = \E[t]{ f_{\theta]}(a, b)^2 + W_{t+1}^2 } \le 1 + \sigma_w^2
\]
and the fact $0 \le {\hat{X}_{t, a}}/{X}_{t, a} \le 1/(1-\gamma_t)$ for all $a \in \mathcal{A}$,
\begin{align*}
  & \E[t]{R_{t+1, A_t, B_t}^2
    \left( \frac{1}{X_{t, A_t}^2}
    - \frac{2 \hat{X}_{t, A_t} }{ X_{t, A_t}^2}
     + \abs{\mathcal{A}} \frac{\hat{X}_{t, A_t}^2}{X_{t, A_t}^2}
    \right)} \\
  & = \sum_{a} \frac{\E[t]{R_{t+1,a, B_t}^2}}{X_{t, a}} 
  + \sum_{a} \frac{\hat{X}_{t, a}}{ {X}_{t, a} } 
  \E[t]{R_{t+1, A_t, B_t}^2} 
  \left(\abs{\mathcal{A}} \hat{X}_{t,a} - 2 \right)\\
  & \le (1+ \sigma_w^2) \left( \sum_{a} \frac{1}{{X}_{t, a}} 
  + \sum_{a} \frac{\hat{X}_{t, a}}{ {X}_{t, a} } \left( \abs{\mathcal{A}} \hat{X}_{t, a} - 2 \right) \right)\\
  & \le (1+ \sigma_w^2) \left( \sum_{a} \frac{1}{X_{t, a}} 
  + \sum_{a} \frac{\hat{X}_{t, a}}{{X}_{t, a}} \left(\abs{\mathcal{A}} - 2 \right) \right) \\
  & \le (1+ \sigma_w^2) \left( \sum_{a} 
  \frac{\abs{\mathcal{A}}}{\gamma_t} + 
  \min( \frac{\abs{\mathcal{A}}}{\gamma_t}, \frac{\abs{\mathcal{A}}}{1- \gamma_t}) (\abs{\mathcal{A}}-2) \right)
  \le \frac{2 (1+ \sigma_w^2) \abs{\mathcal{A}}^2}{\gamma_t}
\end{align*}
Then, we derive one important relationship
\begin{align*}
    \E[0]{ \sqrt{ \sum_a \left(\tilde{C}^+_{T,i} \right)^2} }\leq \E[0]{ \sqrt{ \sum_{t=1}^T \sum_a \left( \tilde{\Re}_{t, a} \right)^2 } } \le \sqrt{ \sum_{t=1}^T \frac{2 (1+ \sigma_w^2) \abs{\mathcal{A}}^2 }{\gamma_t} }
\end{align*}

\paragraph{Step 3 (Put all together.)}
\begin{align*}
    \E[0]{ \sum_{t=1}^T \Re_{t}(a) } = \E[0]{ \sum_{t=1}^T \E[t]{\tilde{\Re}_{t, a}} + \sum_{t=1}^T 2\gamma_t }
    = \E[0]{ \tilde{C}_{T, a} + \sum_{t=1}^T 2\gamma_t } \le \sqrt{ \sum_{t=1}^T \frac{2 (1+ \sigma_w^2)\abs{\mathcal{A}}^2 }{\gamma_t} } + 2\gamma_t
\end{align*}
When $\gamma_t= \gamma$,
\begin{align*}
\E[0]{ \sum_{t=1}^T \Re_{t}(a) } \leq \sqrt{T}\sqrt{\frac{2 (1+ \sigma_w^2) \abs{\mathcal{A}}^2}{\gamma}} + 2 \gamma T ,
\end{align*}
Taking $\gamma = \sqrt[3]{((1+ \sigma_w^2)\abs{\mathcal{A}}^2)/{2T}}$, we have
\begin{align*}
    \E[0]{ \sum_{t=1}^T \Re_{t}(a) } \le 2^{4/3}(1+ \sigma_w^2)^{1/3} \abs{\mathcal{A}}^{2/3} T^{2/3}.
\end{align*}
\end{proof}

\section{Failure analysis for Thompson sampling in \cref{sec:divergence}}
\label{sec:failure_analysis}
\subsection{Basic setting of the counter example} 
Consider a class of matrix games with a payoff matrix $\theta$ defined as 
\begin{equation}
\theta=\begin{bmatrix} 
1 & 1-\Delta \\
1-\Delta & 1 
\end{bmatrix},
\end{equation} 
where $\Delta\in(0,1)$. At time $t$, Alice plays action $A_t \sim P_{X_t}$, and Bob is a best response player who can observe $X_t$ and play the best response strategy $Y_t = \max_{y} X_t^\top \theta y$ by selecting action $B_t \sim P_{Y_t}$. Alice receives the noiseless reward $R_{t+1, A_t, B_t} = \theta_{A_t, B_t}$. 
From the example $\theta$, we can observe the following:
\begin{itemize}
    \item Observation $1$: When Alice uses a pure strategy, she suffers a regret of $\Delta$ at that round due to Bob's best-response strategy.
    \item Observation $2$: The best-response strategy for a uniform strategy is also a uniform strategy. 
\end{itemize} 
Let's define the following terms:
\begin{itemize}
    \item $X_t$ and $reg_t$: Alice's strategy and instantaneous regret at time $t$.
    \item $\Tilde{R}_t$ and $m_t$: $\Tilde{R}_t = \left[ \Tilde{R}_{t1}, \Tilde{R}_{t2} \right]$ is the estimated reward vector by TS-RM estimator; $m_t = \Tilde{R}_{t1}-\Tilde{R}_{t2}$ represents the difference between two rewards. 
\end{itemize} 
Now, let's consider the following remark regarding initialization:
\begin{remark}[Initialization]\label{rem:rm01}
The TS-RM algorithm for Alice, initialized with a uniform strategy, will always result in a pure strategy in $X_2$. 
\end{remark}
\begin{proof} 
In the regret-matching algorithm, the instantaneous regret $reg_1$ at $t=1$ can be represented as
\begin{equation}
    reg_1 = \Tilde{R}_1-\Tilde{R}_1^TX_1 \cdot \mathbf{1},
\end{equation}
Since Alice and Bob are both initialized with a uniform strategy, i.e., $X_1=[0.5, 0.5]$, it can be observed that if $m_1\neq0$, the two elements in $reg_1$ will always have opposite signs. If $m_1=0$, $X_2=X_1$ and can still be regarded as an initialization step.According to the regret-matching updating rule, $X_2\propto\max(reg_1, 0)$, which means $X_2$ must be a pure strategy.
\end{proof} 
Based on Remark~\ref{rem:rm01}, we can draw the following conclusions regarding the counter example:
\begin{itemize}
    \item The choice of uniform initialization for the TS-RM algorithm does not affect the divergence result. 
    \item This result holds regardless of the specific action chosen by Alice at $t=1$, indicating that two symmetric conditions arise depending on whether $X_2=[1,0]$ or $X_2=[0,1]$. 
\end{itemize}

\subsection{TS-RM suffers linear regret}\label{apppro:tsrm} 
According to \cref{rem:rm01}, without loss of generality (w.l.o.g.), let us define the following events: 
\begin{itemize}
    \item Event $\omega_t$: Alice picks the second row and the best-response opponent chooses the first column at time $t$. 
    \item Event $\Omega_t$: Alice picks the second row and the best-response opponent chooses the first column for all time $t^\prime\leq t$. 
\end{itemize} 
The occurrence of event $\Omega_t$ implies that Alice experiences linear regret until time $t$. However, the actual convergence probability is greater than $\prob(\Omega_t)$ since even if $\Omega_t$ does not occur (i.e., Alice occasionally chooses the optimal result $1$), there is still a probability that TS-RM fails. Quantifying this probability is challenging. If we can demonstrate that $\Omega_t$ occurs with a constant probability $c$, then the divergence probability of TS-RM should be greater than $c$. Specifically, due to the symmetric property of the example $\theta$, we obtain the following propositions:
\begin{proposition} 
If Alice initializes with a uniform strategy,
\begin{align*}
\Re(T) \ge 2 \prob(\Omega_T) \Delta\cdot T
\end{align*}
If $\Omega_t$ happens with constant probability for all $t \ge 1$, then Alice suffers linear regret. 
\end{proposition}

\begin{proposition}[Failure of TS-RM]\label{prop:tsrm_appendix}
Suppose Alice initializes with a uniform strategy and utilizes Regret Matching with Thompson Sampling estimator (TS-RM). For any $\Delta \in (0, 1)$ and $\sigma_n > 0$, there exists a constant $c(\Delta, \sigma_n) > 0$ such that for all rounds $t \ge 1$, $\prob(\Omega_t) \ge c(\Delta, \sigma_n)$.
\end{proposition}
To investigate the divergence behavior of the TS-RM algorithm, an experiment is conducted using the counter example in . Different values of $\Delta$ and $\sigma_n^2$ are considered, and $200$ independent simulation runs are performed for each combination. The averaged divergence results across these runs are shown in Fig.\ref{fig:divts}, which illustrates that the probability of divergence decreases as $\Delta$ and $\sigma_n^2$ increase, consistent with our proposition (Prop.\ref{prop:tsrm_appendix}). In the following, we provide a detailed proof for the divergence of the TS-RM algorithm. 
\begin{figure}
    \centering
    \includegraphics[width=.4\linewidth]{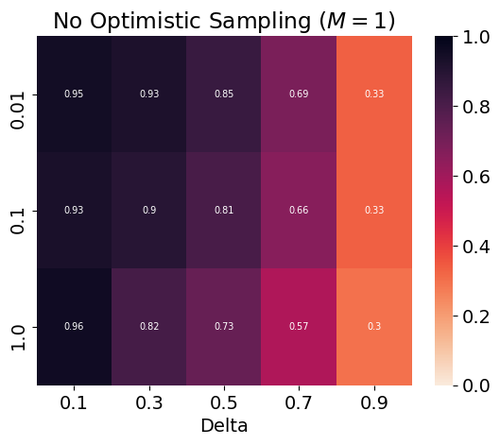}
    \caption{Divergence probability for TS-RM}
    \label{fig:divts}
\end{figure}

\begin{proof} 
Our goal is to prove that there exists a constant $c \geq 0$ such that $\prob(\Omega_t)\geq c$ for all $t$. The probability $\prob(\Omega_t)$ can be expressed as
\begin{equation}  
    \prob(\Omega_t) = \prob(\omega_1)\prob(\omega_2|\Omega_1)\ldots\prob(\omega_t|\Omega_{t-1}) 
\end{equation}  
Referring to~\cref{rem:rm01}, $X_2$ is a pure strategy. Since Alice chooses the second row according to $\Omega_2$, we get $X_2=[0,1]$ and $m_1<0$. Following the event $\Omega_t$, we get $X_t=[0,1], \forall t\geq2$, which indicates that for the TS-RM estimator $\Tilde{R}_t$, only the posterior distribution of $\Tilde{R}_{t2}$ is updated. Since no noise is considered in received rewards, by Bayesian rule, we have
$$\Tilde{R}_t=\left[z_t, \frac{t}{t+\sigma_n^2}(1-\Delta)+\sqrt{\frac{\sigma_n^2}{\sigma_n^2+t}}z^{\prime}_t\right],$$
where $z_t, z^{\prime}_t\sim\mathcal{N}(0, 1)$ are two independent \textit{r.v.}s, and $\sigma_n^2$ is the noise variance. Therefore, we can express the regret as: 
\begin{equation}
    reg_t = \Tilde{R}_t - X_t^{T}\Tilde{R}_t\cdot\mathbf{1} = [m_t, 0], \quad \forall t\geq2 
\end{equation}
where $ m_t = z_t - \frac{t}{t+\sigma_n^2}(1-\Delta)-\sqrt{\frac{\sigma_n^2}{t+\sigma_n^2}}z^{\prime}_{t} = z_t-\sqrt{\frac{\sigma_n^2}{t+\sigma_n^2}}z^{\prime}_t-\frac{t}{t+\sigma_n^2}(1-\Delta)$. 

The cumulative regret can be represented as
\begin{equation} 
    \Re_t = [0.5m_1+\sum_{k=2}^{t}m_k, -0.5m_1] 
\end{equation} 
According to updating rule of TS-RM, we have:
\begin{equation}
\begin{aligned}
\prob(\omega_t|\Omega_{t-1}) &=\prob(0.5m_1+\sum_{k=2}^{t}m_k\leq 0|\Omega_{t-1}) \\ 
&\geq \prob(\sum_{k=2}^{t}m_k\leq0|\Omega_{t-1}) \\ 
&=\prob(\sum_{k=2}^{t}m_k\leq0|\Omega_{t-1})\left(\prob(\Omega_{t-1})+\prob(\bar{\Omega}_{t-1})\right) \\
&\geq \prob(\sum_{k=2}^{t}m_k\leq0|\Omega_{t-1})\prob(\Omega_{t-1}) + \prob(\sum_{k=2}^{t}m_k\leq0|\bar{\Omega}_{t-1})\prob(\bar{\Omega}_{t-1}) \\
&=\prob(\sum_{k=2}^{t}m_k\leq0) \\
\end{aligned}
\end{equation}
where the first inequality is because $m_1\leq0$ (conditioned on $\Omega_{t-1}$), and the second one is due to $\prob(\sum_{k=2}^{t}m_k\leq0|\Omega_{t-1})\geq \prob(\sum_{k=2}^{t}m_k\leq0|\bar{\Omega}_{t-1})$.

Define 
$$N_t = \sum\limits_{k=2}^{t}m_k\sim \mathcal{N}\left(-\sum\limits_{k=2}^{t}\frac{k}{k+\sigma_n^2}(1-\Delta),\  \sum\limits_{k=2}^{t}(1+\frac{\sigma_n^2}{k+\sigma_n^2})\right)\triangleq \mathcal{N}(\mu_t, \sigma_t^2).$$ 
As $t\to\infty$, we have:
\begin{equation} 
\begin{aligned}  
\lim_{t\to\infty}\log\prob(\Omega_t) &= \log\prod_{t=1}^{\infty}\prob(\omega_t|\Omega_{t-1})=\sum_{t=1}^{\infty}\log\prob(\omega_t|\Omega_{t-1}) \\
&\geq \sum_{t=1}^{\infty}\log\prob(N_t\leq 0) \\ 
&= \sum_{t=1}^{\infty}\log\prob(\mu_t+\sigma_tZ\leq0), \quad Z\sim\mathcal{N}(0,1) \\ 
&= \sum_{t=1}^{\infty}\log\prob(Z\leq -\frac{\mu_t}{\sigma_t}) \\
&= \sum_{t=1}^{\infty}\log \Phi(-\frac{\mu_t}{\sigma_t}) \\
&=\sum_{t=1}^{\infty}\log\Phi \left( \frac{\sum_{k=2}^{t}\frac{k}{k+\sigma_n^2}(1-\Delta)}{\sqrt{\sum_{k=2}^{t}(1+\frac{\sigma_n^2}{k+\sigma_n^2})}} \right) \\
&\geq \sum_{t=1}^{\infty}\log\Phi \left( \frac{(1-\Delta)\left(t-\sigma_n^2\ln{(t+\sigma_n^2)}+2\sigma_n^2\ln{\sigma_n}-1/(\sigma_n^2+1)\right)}{\sqrt{t+\sigma_n^2\ln{(t+\sigma_n^2)}-2\sigma_n^2\ln{\sigma_n}-(\sigma_n^2+2)/(\sigma_n^2+1)}} \right) \\ 
\end{aligned}  
\end{equation} 
where $\Phi(\cdot)$ is the cumulative distribution function of the standard normal distribution, and the last inequality is due to $\sum\limits_{k=1}^{t}\frac{1}{k+a}\leq \int_0^{t}\frac{1}{k+a}dk$. Define
\begin{equation}\label{equ:tsrm01}
   f_t(\Delta,\sigma_n) = \frac{(1-\Delta)\left(t-\sigma_n^2\ln{(t+\sigma_n^2)}+2\sigma_n^2\ln{\sigma_n}-1/(\sigma_n^2+1)\right)}{\sqrt{t+\sigma_n^2\ln{(t+\sigma_n^2)}-2\sigma_n^2\ln{\sigma_n}-(\sigma_n^2+2)/(\sigma_n^2+1)}},
\end{equation}
Referring to the lower bound of the standard Gaussian distribution, we can continue to derive: 
\begin{equation}\label{equ:tsrm02}
\begin{aligned}
\lim_{t\to\infty}\log\prob(\Omega_t) &\geq \sum_{t=1}^{\infty}\log\Phi(f_t(\Delta,\sigma_n)) \\ 
&\geq \sum_{t=1}^{\infty}\log\left(1-\frac{1}{\sqrt{2\pi}f_t(\Delta,\sigma_n)e^{f_t^2(\Delta,\sigma_n)/2}} \right) \\ 
& \geq \sum_{t=1}^{\infty} \left( -\frac{1}{\sqrt{2\pi}f_t(\Delta,\sigma_n)e^{f_t^2(\Delta,\sigma_n)/2}} \right) \\ 
&> -\infty, 
\end{aligned}
\end{equation}
This shows that there exists a constant $c^{\prime}>0$ such that:
\begin{equation}
\lim_{t\to\infty}\log\prob(\Omega_t)\geq\log c^{\prime} >-\infty 
\end{equation}
In other words, we have:
\begin{equation} 
\lim_{t\to\infty}\prob(\Omega_t)\geq c > 0,
\end{equation}
where $c=e^{c^{\prime}}$. Moreover,  $\prob(\Omega_t)\geq\lim\limits_{t\to\infty}\prob(\Omega_t)\geq c$ for a finite sequence. Specifically, let $\Delta=0.1$ and $\sigma_n^2=0.1$, we can get $c^{\prime}=-0.62$, and $c=0.54$. 
\end{proof} 
Moreover, the function $f_t(\Delta, \sigma_n^2)$, combined with the derivation above, demonstrates that the divergence probability $\prob(\Omega_t)$ decreases as $\Delta$ and $\sigma_n^2$ increase, which is consistent with our proposition (Proposition~\ref{prop:tsrm_appendix}) and the experiments. 

The above argument is based on the frequentist setting where the underlying instance is fixed, and the agent does not access the right noise likelihood function of the environment. We conjecture that under the Bayesian setting where the prior and likelihood in the game environment are available to the agent, with the exact Bayes posterior, the TS-RM still suffers linear Bayesian adversarial regret.
\subsection{Why optimistic variant of TS would not suffer linear regret?}\label{apppro:otsrm} 
Assume that Alice chooses the wrong action until time $t$. By Prop~\ref{prop:tsrm_appendix}, the TS-RM algorithm will continue to choose the wrong action with a constant probability $c$. Different from TS-RM, we will prove that even if Alice chooses the wrong action until time $t$, OTS-RM will eventually yield a sub-linear regret with high probability.
\begin{proof}
Unlike TR-RM, the OTS-RM algorithm uses $M$ samples for optimistic sampling in each round. Under the assumption that Alice takes the wrong action until time $t$, we have:
\begin{equation}
\begin{cases}
x_{ti}&\sim\mathcal{N}_i(0,1),\quad i=1,\ldots,M\\
y_{ti}& \sim\mathcal{N}_i\left(\frac{t}{t+\sigma_n^2}(1-\Delta), \frac{\sigma_n^2}{\sigma_n^2+t}\right),\quad i=1,\ldots,M
\end{cases}
\end{equation}
Let $\Tilde{R}_{t1}=\max\limits_{i}x_{ti}$ and $\Tilde{R}_{t2}=\max\limits_{i}y_{ti}$, which are just $\Tilde{R}_{t}(1nd)$ and $\Tilde{R}_{t}(2nd)$ mentioned above, respectively.

The proof will first show that $\Tilde{R}_{t1}\geq \Tilde{R}_{t2}$ with high probability. As a result, $\Re_{t1}$ will decrease, and eventually Alice's strategy will change from a pure strategy $[0,1]$ to a mixed strategy, indicating a decay of $\prob(\Omega_t)$ over time $t$.

According to the anti-concentration property in~\cref{lem:normal_optimism}, we have:
\begin{equation}\label{appequ: anticb}
\prob(\max_iy_{ti}\leq\frac{t}{t+\sigma_n^2}(1-\Delta)+\sqrt{\frac{2\sigma_n^2\log(M/\delta_1)}{t+\sigma_n^2}}) \geq 1-\delta_1
\end{equation}
Thus, the first step in the proof can be written as:
\begin{equation}
\begin{aligned}
\prob(\Tilde{R}_{t1}\geq \Tilde{R}_{t2}) &= \prob(\max_{i}x_i\geq\max_{i}y_i) \\ 
&\geq (1-\delta_1)\prob(\max_{i}x_i\geq\max_{i}y_i\mid \epsilon) \\ 
&\geq (1-\delta_1)\prob(\max_ix_{ti}\geq\frac{t}{t+\sigma_n^2}(1-\Delta)+\sqrt{\frac{2\sigma_n^2\log(M/\delta_1)}{t+\sigma_n^2}}\mid \epsilon) \\ 
&\geq 1-\delta_1-(1-\delta_1)\Phi^{M}\left(\frac{t}{t+\sigma_n^2}(1-\Delta)+\sqrt{\frac{2\sigma_n^2\log(M/\delta_1)}{t+\sigma_n^2}}\right) 
\end{aligned}
\end{equation}
Here, $\epsilon$ represents the event where the anti-concentration property occurs. 

Let $f(t,\Delta,\sigma_n)=\frac{t}{t+\sigma_n^2}(1-\Delta)+\sqrt{\frac{2\sigma_n^2\log(M/\delta_1)}{t+\sigma_n^2}}$. Then, we have:
\begin{equation}\label{equ:ostconv}
\begin{aligned}
\prob(\Tilde{R}_{t1}\geq \Tilde{R}_{t2}) &\geq 1-\delta_1-(1-\delta_1)\Phi^{M}(f) \\ 
&\geq 1-\delta_1-(1-\delta_1)\left(1-\frac{f}{\sqrt{2\pi}(f^2+1)e^{f^2/2}} \right)^M\\
&\geq 1-\delta_1-(1-\delta_1)\exp{\left(\frac{-Mf}{\sqrt{2\pi}(f^2+1)e^{f^2/2}}\right) } \\ 
\end{aligned}    
\end{equation}   
where $\Phi(x)\leq 1-\frac{x}{\sqrt{2\pi}(x^2+1)e^{x^2/2}}$~\citep{gordon1941values} and $(1-x)^M\leq e^{-Mx}$.
\end{proof}
To obtain further insights into the relationship between $M$ and $t$, we have depicted a figure in Figure \ref{fig:otsconv} that corresponds to the inequality in Equation \ref{equ:ostconv}. Based on the analysis, we draw the following conclusions:
\begin{itemize}
\item When $\Delta,\sigma_n$ and $t$ are fixed, $\prob(\Tilde{R}_{t1}\geq \Tilde{R}_{t2})$ increases with $M$. Additionally, the value of $t$ has a significant influence on $\prob(\Tilde{R}_{t1}\geq \Tilde{R}_{t2})$.  
\item When $\Delta,\sigma_n$ and $M$ are fixed, the probability $\prob(\Tilde{R}_{t1}\geq \Tilde{R}_{t2})$ increases with time $t$. Moreover, $\prob(\Tilde{R}_{t1}\geq \Tilde{R}_{t2})$ will quickly reach a region close to the maximum in just a few rounds. 
\item When $\sigma_n, M$ and $t$ are fixed, $\prob(\Tilde{R}_{t1}\geq \Tilde{R}_{t2})$ initially increases with $\Delta$ and subsequently decreases with $\Delta$. 
\begin{figure}
    \centering
    \includegraphics[width=.6\linewidth]{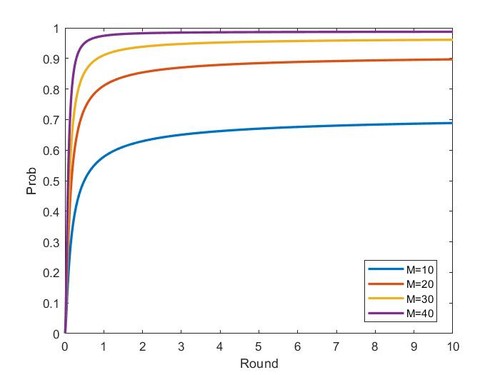}
    \caption{Relationship between $t$ and $M$ (fix $\Delta$ and $\sigma_n$)}  
    \label{fig:otsconv} 
\end{figure}
\end{itemize}
These findings provide valuable insights into the behavior of the OTS-RM algorithm and support our claim that $\prob(\Tilde{R}_{t1}\geq \Tilde{R}_{t2})$ increases as $M$ grows. This increasing probability implies that Alice's strategy will transition from a pure strategy $[0,1]$ to a mixed strategy, indicating a decay of $\prob(\Omega_t)$ over time $t$. Consequently, the algorithm achieves sub-linear regret with high probability.

\section{Technical details in \cref{sec:analysis}}
\label{sec:analysis_details}
Our analysis road map is as follows.
First, in~\cref{sec:general-regret}, as described in \cref{prop:regret-decomposition}, we derive a general regret decomposition in given any imagined reward sequence $\{ \tilde{R}_t: t \in \Z_{++} \}$, where each $\tilde{R}_{t+1}$ is constructed using history information $H_t$ with algorithmic randomness.
Then, to further upper bound the regret, we introduce the generic upper confidence bound (UCB) sequence and lower confidence bound (LCB) sequence as in \cref{def:UCB_seq}. With these sequences, we have a general regret bound in \cref{prop:general-regret-bound} with generic UCB and LCB sequences.
Next, as described in \cref{rem:proof-sketch}, we specify the so-called information-theoretic confidence bound in \cref{sec:info-conf-bound} and show that the imagined reward sequence has good properties when compared with specified sequences in \cref{sec:optimism}, finally yielding the information-theoretic regret upper bound in \cref{sec:estimation-information}.

\subsection{General regret bound}
\label{sec:general-regret}
\begin{proposition}[Restate regret decomposition in \cref{prop:regret-decomposition}]
  For any $a \in \actions$, the one-step regret can be decomposed by
    \begin{align*}
      & \E{ R_{t+1, a, B_t} - R_{t+1, A_t, B_t} \given \theta} = 
      \E{ f_{\theta}(a, B_t) - f_{\theta}(A_t, B_t) \given \theta} \nonumber \\
      & = \underbrace{ \E{ \tilde{R}_{t+1}(a) - \tilde{R}_{t+1}(A_t) \given \theta} }_{(I)} + \underbrace{ \E{ f_{\theta}({a, B_t} ) - \tilde{R}_{t+1}( a ) \given \theta} }_{(II)}
      + \underbrace{ \E{ \tilde{R}_{t+1}(A_t) - f_{\theta}(A_t, B_t) \given \theta} }_{(III)}
    \end{align*}
\end{proposition}
\begin{proof}
    Since
\begin{align*}
  \E{ R_{t+1, a, B_t} - R_{t+1, A_t, B_t} \given \theta }
  = \E{ f_{\theta}(a, B_t) - f_{\theta}(A_t, B_t) \given \theta }
\end{align*}
We have introduced the imagined time-varying sequence $\tilde{R} := \{ \tilde{R}_{t} : t \in \Z_{++} \}$, where each $\tilde{R}_t$ is constructed using history information $H_t$ and takes value in $[0, 1]^{\actions}$.
\begin{align}
  \label{eq:general-decomposition-a}
  f_{\theta}(a, B_t) - f_{\theta}(A_t, B_t)
  = \underbrace{ \tilde{R}_{t+1}( a ) - \tilde{R}_{t+1}(A_t) }_{(I)~\operatorname{adv}_{t+1}(a) } + 
  \underbrace{ f_{\theta}(a, B_t) - \tilde{R}_{t+1}(a) }_{(II)~\operatorname{pess}_{t+1}(a) }
  + \underbrace{ \tilde{R}_{t+1}(A_t) - f_{\theta}( A_t, B_t ) }_{(III) \est_{t+1} }
\end{align}
\end{proof}

\paragraph*{Reduction to full-information adversarial regret.}
Any algorithm $\pi = \pi^{\operatorname{adv-est}}$ constructs the imagined reward sequence $\tilde{R}$ and use $\operatorname{adv}$ algorithm for no-regret update will lead to,
\begin{align}
\label{eq:reduction-to-full-adv}
& \sum_{t=0}^{T-1} \E{ f_{\theta}(a, B_t) - f_{\theta}(A_t, B_t) \given \theta } \nonumber \\
& \le
\Re_{\operatorname{full}}(a; T, \operatorname{adv}, \tilde{R}) + \sum_{t=0}^{T-1} \E{ \operatorname{pess}_{t+1}(a) \mid \theta} + \sum_{t=0}^{T-1}\E{ \est_{t+1} \given \theta }
\end{align}
where
\begin{align*}
  \Re_{\operatorname{full}}(a; T, \operatorname{adv}, \tilde{R} ) & = \E{ \sum_{t=0}^{T-1} \tilde{R}_{t+1}(a) - \tilde{R}_{t+1} (A_t) }.
\end{align*}
Recall the Bayesian adversarial regret is
\begin{align*}
    \Re(T, \pi^{\operatorname{alg}}, \pi^B) 
    = \E{ \Re( T, \pi^{\operatorname{alg}}, \pi^B, \theta ) }
\end{align*}
where the expectation is taken over the prior distribution of $\theta$.
From \cref{eq:reduction-to-full-adv}, we have
\begin{align}
    \Re(T, \pi^{\operatorname{alg}}, \pi^B) 
    & = \E{ \max_{a \in \mathcal{A}} \sum_{t=0}^{T-1} \E{ R_{t+1, a, B_t} - R_{t+1, A_t, B_t} \mid \theta } } \nonumber \\
    & \le \max_{a \in \actions} \Re_{\operatorname{full}}(a; T, \adv, \tilde{R} ) +
    \E{ \max_{a \in \actions} \sum_{t=0}^{T-1} \E{ \operatorname{pess}_{t+1}(a) \mid \theta } } + \sum_{t=0}^{T-1}\E{\est_{t+1} } \nonumber \\
    & = \Re_{\operatorname{full}}(T, \adv, \tilde{R} ) +
    \E{ \max_{a \in \actions} \sum_{t=0}^{T-1} \E{ \operatorname{pess}_{t+1}(a) \mid \theta } } + \sum_{t=0}^{T-1}\E{\est_{t+1} }
    \label{eq:decomposition-Bayes-adv}
\end{align}

\begin{remark}
  If each imagined reward $\tilde{R}_{t}$ in the sequence $\tilde{R}$ takes bounded value in $[-c, c]^{\actions}$,
  any suitable full information adversarial algorithm $\adv$ will give satisfied bound for $\Re_{\operatorname{full}}(a; T, \operatorname{adv}, \tilde{R} )$ for any $a \in \actions$.
  Specifically, Hedge suffers $\Re_{\operatorname{adv}}(a; T, \text{Hedge}, \tilde{R} ) = \mathcal{O}( 2c \sqrt{T \log \actions} )$
  and RM suffers $\Re_{\operatorname{full}}(a; T, \text{RM}, \tilde{R} ) = \mathcal{O}( 2c \sqrt{T \actions} )$.
\end{remark}
Next, we focus on the pessimism term and estimation term.
We now focus on the case generalized from the optimistic Thompson sampling.
\begin{assumption}[Restriction on the imagined reward sequence $\tilde{R}$]
  \label{asmp:res_imagine_R}
  In the following context, the imagined reward sequence $\tilde{R} = (\tilde{R}_{1}, \ldots, \tilde{R}_{t+1}, \ldots )$ satisfies that
  (1) $\tilde{R}_{t+1}(a) \in [0, C]$ for all $a \in \actions$ and
  (2) $\tilde{R}_{t+1}(a)$ is random only through its dependence on $Z_{t+1}$ given the history $H_t, B_t$ for all $a \in \actions$.
  To clarify, $\tilde{R}_{t+1}(a)$ has no dependence on $A_t$ and $\tilde{R}_{t+1, A_t, B_t}$.
\end{assumption}
\begin{definition}[UCB and LCB sequence]
  \label{def:UCB_seq}
    UCB sequence $U = ( U_t \given t \in \NN )$ LCB sequence $L = ( L_t \given t \in \NN )$ are two sequences of functions where each $0 \le L_t \le U_{t} \le C$ are both deterministic given history $H_{t}$.
\end{definition}
\begin{definition}[Optimistic Event]
\label{def:optimistic-event}
  For any imagined reward sequence $\tilde{R}_{t+1}$ in \cref*{asmp:res_imagine_R}
  and any upper confidence sequence $U$ in \cref{def:UCB_seq},
  we define the event 
  \[
    \event{E}_{t}(\tilde{R}, U, B_t) := \{ \tilde{R}_{t+1}(a) \ge U_{t}(a, B_t), \forall a \in \mathcal{A} \}.
  \]
\end{definition}
\begin{fact}
  The event $\event{E}^o_{t}(\tilde{R}, U, B_t)$ is random only through its dependence on $Z_{t+1}$ given $H_t, B_t$.
  The pessimism term can be decomposed according to the event $\event{E}^o_{t}(\tilde{R}, U, B_t)$ 
\begin{align*}
  \tilde{R}_{t+1}(a) - f_{\theta}(a, B_t)
  = \1{ \event{E}^o_t(\tilde{R}, U, B_t ) } ( \tilde{R}_{t+1}(a) - f_{\theta}(a, B_t) )
  + 
  (1 - \1{ {\event{E}}^o_t(\tilde{R}, U, B_t )} ) ( \tilde{R}_{t+1}(a) - f_{\theta}(a, B_t) )
\end{align*}
Consider the case where $f_{\theta}$ takes values in $[0, C]$ and  $\tilde{R}_{t+1} \in [0, C]$ by \cref*{asmp:res_imagine_R}, for all $a \in \mathcal{A}$,
\begin{align}
  \label{eq:pess}
    f_{\theta}(a, B_t) - \tilde{R}_{t+1}(a)
  \le \1{ \event{E}^o_t(\tilde{R}, U, B_t ) } ( f_{\theta}(a, B_t) - U_{t}(a, B_t))
  + 
  C (1 - \1{ {\event{E}}^o_t(\tilde{R}, U, B_t) } ).
\end{align}
\end{fact}
\begin{definition}[Concentration Event]
\label{def:concentration-event}
  For any imagined reward sequence $\tilde{R}_{t+1}$ in \cref*{asmp:res_imagine_R}
  and any upper confidence sequence $U'$ in \cref*{def:UCB_seq},
  we define the event 
  \[
    \event{E}^c_{t}(\tilde{R}, U', A_t, B_t) :=  \{ \tilde{R}_{t+1}(A_t) \le U'_{t}(A_t, B_t) \}.
  \]
\end{definition}
\begin{fact}
  Consider the case where $f_{\theta}$ takes values in $[0, C]$ and  $\tilde{R}_{t+1} \in [0, C]$ by \cref*{asmp:res_imagine_R}, the event $\event{E}_{t}(\tilde{R}, U', A_t, B_t)$ is random only through its dependence on $Z_{t+1}$ given $H_t, A_t, B_t$. The estimation term then becomes,
\begin{align}
  \label{eq:estimation}
  \tilde{R}_{t+1}(A_t) - f_{\theta}(A_t, B_t) \le 
   ( \1{ \event{E}^c_t( \tilde{R}, U', A_t, B_t ) } ) ( U'_{t}(A_t, B_t) - f_{\theta}(A_t, B_t) ) +
  C (1 - \1{ \event{E}^c_t( \tilde{R}, U', A_t, B_t ) } ).
\end{align}
\end{fact}

\begin{definition}[Confidence event]
\label{def:confidence-event}
  Define the confidence event at round $t$ as
  \begin{align*}
    \event{E}_t(f_{\theta}, B_t) := \{ \forall a \in \mathcal{A}, f_{\theta}(a, B_t) \in [L_t(a, B_t), U_t(a, B_t)] \}.
  \end{align*}
\end{definition}
\begin{fact}
  The event $\mathcal{E}^c_t(f_{\theta}, B_t)$ is deterministic conditioned on history $H_t$, $B_t$ and $\theta$.
\end{fact}
Based on the definition of the confidence event, the pessimism term from \cref{eq:pess} becomes
\begin{align*}
  f_{\theta}(a, B_t) - \tilde{R}_{t+1}(a)
  \le C ( 1 - \1{\event{E}^o_t(\tilde{R}, U, B_t) \cap \event{E}^c_t(f_{\theta}, B_t) } ), \quad \forall a \in \actions.
\end{align*}
The estimation term from \cref{eq:estimation} becomes
\begin{align*}
  \tilde{R}_{t+1}(A_t) - f_{\theta}(A_t, B_t)
  & \le \1{\event{E}^c_t(\tilde{R}, U', A_t, B_t)\cap \event{E}^c_t(f_{\theta}, B_t)} (U'_t(A_t, B_t) - L_t(A_t, B_t)) \\
  & \quad
  + C \left(1 - \1{\event{E}^c_t(\tilde{R}, U, A_t, B_t)} \1{\event{E}^c_t(f_{\theta}, B_t)} \right)
\end{align*}
Denote the complement event of $\mathcal{E}$ as $\neg {\mathcal{E}}$.
\begin{align}
\label{eq:upper-pess}
    \E{ \max_{a \in \actions} \sum_{t=0}^{T-1} \E{ \operatorname{pess}_{t+1}(a) \mid \theta } } 
    & \le 
    \E{ \max_{a \in \actions} \sum_{t=0}^{T-1} \E{ C ( 1 - \1{\event{E}^o_t(\tilde{R}, U, B_t) \cap \event{E}^c_t(f_{\theta}, B_t) } ) \mid \theta } } \nonumber \\
    & = \E{ \sum_{t=0}^{T-1} \E{ C ( 1 - \1{\event{E}^o_t(\tilde{R}, U, B_t) \cap \event{E}^c_t(f_{\theta}, B_t) } ) } } \nonumber \\
    & = C \sum_{t=0}^{T-1} \prob\left( \neg {\event{E}^o_t(\tilde{R}, U, B_t)} \cup \neg { \event{E}^c_t(f_{\theta}, B_t) } \right)
\end{align}
and by assuming $U'_t \ge L_t$
\begin{align}
\label{eq:upper-est}
    \E{ \sum_{t=0}^{T-1} \est_{t+1} } 
    & \le \sum_{t=0}^{T-1} \E{ (U'_t(A_t, B_t) - L_t(A_t, B_t) ) } \nonumber \\
    & + C \sum_{t=0}^T \prob( \neg \event{E}^c_t(\tilde{R}, U', A_t, B_t) \cup \neg \event{E}^c_t(f_{\theta}, B_t))
\end{align}
\begin{proposition}[General Regret Bound with confidence sequence.]
\label{prop:general-regret-bound}
    Given sequences $U' \ge U \ge L$, we upper bound the Bayesian adversarial regret with
    \begin{align}
        \Re(T, \pi^{\operatorname{alg}}, \pi^B) 
        & \le  \Re_{\full}(T, \adv, \tilde{R}) + \sum_{t=0}^{T-1} \E{ (U'_t(A_t, B_t) - L_t(A_t, B_t) ) } \nonumber \\
        & \quad + C \sum_{t=0}^{T-1} \prob( \neg \event{E}^c_t(\tilde{R}, U', A_t, B_t) ) + 2 \prob( \neg \event{E}^c_t(f_{\theta}, B_t)) + \prob\left( \neg {\event{E}^o_t(\tilde{R}, U, B_t)} \right)
    \end{align}
\end{proposition}
\begin{proof}
    This is a direct consequence of \cref{eq:decomposition-Bayes-adv,eq:upper-pess,eq:upper-est}.
\end{proof}
\begin{remark}
\label{rem:proof-sketch}
    In the following sections, we will discuss the specific choice of the UCB and LCB sequence $U', U, L$. 
In \cref{sec:info-conf-bound}, we will show that with information-theoretic confidence bound, the probability $\prob( \neg \event{E}^c_t(f_{\theta}, B_t) )$ of the function $f_{\theta}$ not covered in the confidence region is small.
In \cref{sec:optimism}, we will show the imagined reward sequence constructed by OTS and UCB will lead to small will stay within the sequence $U'$ and $U$ with high probability, i.e., the probability $\prob\left( \neg {\event{E}^o_t(\tilde{R}, U, B_t)} \right)$ and $\prob( \neg \event{E}^c_t(\tilde{R}, U', A_t, B_t) )$ is small enough.
In \cref{sec:estimation-information}, we show that $\sum_{t=0}^{T-1} \E{ (U'_t(A_t, B_t) - L_t(A_t, B_t) ) }$ can be bounded by the mutual information $I(\theta; H_T)$ with the information-theoretic confidence bound defined in \cref{sec:info-conf-bound}.
\end{remark}

\subsection{Information-theoretic confidence bound}
\label{sec:info-conf-bound}
\begin{fact}[Chernoff bound]
  \label{fact:gaussian_chernoff}
  Suppose $X$ is normal distributed $N(\mu, \sigma^2)$, the optimized chernoff bound for $X$ is
  \begin{align*}
    \prob( X - \mu \ge c ) & \le \min_{t>0} \frac{\exp( \sigma^2 t^2 /2 )}{ \exp( tc ) } = \exp( - c^2 / \sigma^2 )
  \end{align*}
\end{fact}
\begin{lemma}
\label{lem:cb}
Conditioned on $H_t$ and $B_t$, 
Define the set $\mathcal{F}_t$ as
\begin{align*}
  \mathcal{F}_t := \left\{ f_{\theta} : \abs{ f_{\theta}(a, B_t) - \mu_{t}(a, B_t) } \le \sqrt{\beta'_t} \sigma_{t}(a, B_t), \forall a \in \actions \right\}
\end{align*}
Then,
\[
  \prob( f_{\theta} \in \mathcal{F}_t ) \ge 1 -  2 \actions\exp( - \beta'_t / 2).
\]
\end{lemma}
\begin{proof}

Since $f_{\theta}(a, b) \mid H_t$ is distributed as $N(\mu_t(a, b), \sigma_t(a, b) )$, by \cref{fact:gaussian_chernoff},
\begin{align}
  \label{eq:con_f}
  \prob( \abs{ f_{\theta}(a, b) - {\mu_t(a, b)} } \ge \sqrt{\beta'_t} \sigma_{t}(a, b) \mid H_t )
  & \le 2 \exp \left( - \frac{\beta'_t}{ 2 } \right)
\end{align}
By union bound, we have
\begin{align*}
  \prob( \abs{ f_{\theta}(a, b) - {\mu_t(a, b)} } \ge \sqrt{\beta'_t} \sigma_{t}(a, b) , \forall a \in \actions \mid H_t ) \le 2 \actions \exp( - \beta'_t / 2 ),
\end{align*}
for any fixed $b \in \bactions$,
We observe that conditioned on $H_t$, the opponent's action $B_t$ is independent of $f_{\theta}(a, b)$ for all $(a, b) \in \actions \times \bactions$.
Therefore, we further derive
\begin{align*}
      \prob( \abs{ f_{\theta}(a, B_t) - {\mu_t(a, B_t)} } \ge \sqrt{\beta'_t} \sigma_{t}(a, b) , \forall a \in \actions \mid H_t ) \le 2 \actions \exp( - \beta'_t / 2 ).
\end{align*}
Taking expectations on both sides, we prove the lemma.
\end{proof}
Let the UCB and LCB sequences be 
\[
    U = ( \mu_t(a, b) + \sqrt{\beta'_t} \sigma_t(a, b) : t \in \NN), \quad L = ( \mu_t(a, b) - \sqrt{\beta'_t} \sigma_t(a, b): t \in \NN).
\]
Let 
\[
    \beta'_t = 2 \log \actions \sqrt{t}.
\] 
By \cref{lem:cb}, we can see the probability introduced in \cref{def:confidence-event} is $\prob( \neg \event{E}^c_t(f_{\theta}, B_t) ) \le 2\actions\exp( - \beta_t'/2) = 2/\sqrt{t}$. 
\paragraph{Relate $\sigma_t$ to information-theoretic quantity.}
\begin{fact}[Mutual information of Gaussian distribution]
  \label{fact:mutual_gaussian}
  If $f_\theta(a) \sim N(\mu(a), \sigma(a))$ and $o = f_\theta(a) + w$ with fixed $a$ and $w \sim N\left(0, \sigma_{w}^{2}\right)$, then
  \[
  I(\theta ; o)=\frac{1}{2} \log \left(1 + {\sigma_{w}^{-2} \sigma(a) }\right) 
  \]
\end{fact}
\begin{proof}
    \begin{align*}
        I(\theta; o) = H(o) - H(o \mid \theta ) = \frac{1}{2} \log 2 \pi e (\sigma(a)^2 + \sigma_w^2) - \frac{1}{2} \log 2 \pi e \sigma_w^2 = \frac{1}{2} \log ( 1 + \sigma_w^{-2} \sigma(a) )
    \end{align*}
\end{proof}
For an fixed action pair $(a, b)$, by the \cref{fact:mutual_gaussian},
\begin{align*}
  I_{t}\left( \theta ; R_{t+1, A_t, B_t} \mid A_{t} = a, B_t = b \right) =
  \frac{1}{2} \log \left( 1 + \frac{ \sigma^2  _{t}(a, b)  }{ \sigma_{w}^{2} } \right)
\end{align*}
Let the width function $w_t$ be
\[
w_{t}(a, b) = \sqrt{\beta_{t} I_{t} \left( \theta ; R_{t+1, A_t, B_t} \mid A_{t} = a, B_t = b \right)} \quad \text { with } \quad \beta_{t} = 
\frac{2 \beta'_t}{\log (1 + \sigma_w^{-2} ) }.
\]
Expanding $w_t(a, b)$ leads to
\begin{align}
  \label{eq:width}
  w_{t}(a, b)^2 
  = \beta_t I_t( \theta; R_{t+1, A_t, B_t} \given A_t= a, B_t = b )
   = \frac{ \beta'_t \log( 1 + \sigma_w^{-2} \sigma^2_t(a, b) ) }{ \log( 1 + \sigma_w^{-2} ) }
   \ge \beta'_t \sigma^2_t(a, b).
\end{align}
The last inequality follows from the fact that $\frac{x}{\log (1+x)}$ monotonically increases for $x>0$, and that $\sigma^2_{t}(a, b) \le k((a, b), (a, b)) \le 1$, leading to the result
$
  \sigma_{t}^2(a, b) \le \frac{1}{\log(1 + \sigma_w^{-2})} \log(1 + \sigma_w^{-2}\sigma^2_t(a, b)).
$

\subsection{Anti-concentration behavior of optimistic Thompson sampling}
\label{sec:optimism}
Let the UCB sequences be 
\[
    U = ( \mu_t(a, b) + \sqrt{\beta'_t} \sigma_t(a, b) : t \in \NN), \quad U' = ( \mu_t(a, b) + \sqrt{2\log (M \sqrt{t})} \sigma_t(a, b): t \in \NN).
\]
In this section, we show the probability $\prob\left( \neg {\event{E}^o_t(\tilde{R}, U, B_t)} \right)$ and $\prob( \neg \event{E}^c_t(\tilde{R}, U', A_t, B_t) )$ is small enough.
The following two lemmas are important.
\begin{lemma}
\label{lem:normal_optimism}
Consider a normal distribution $N \left(0, \sigma^2 \right)$ where $\sigma$ is a scalar.
Let $\eta_{1}, \eta_{2}, \ldots, \eta_{M}$ be $M$ independent samples from the distribution.
For any $w \in \R_+$,
\begin{align*}
  \prob\left( \max_{j \in [M]} \eta_{j} \ge w  \right) = 1 - \left[\Phi\left( \frac{w}{\sigma} \right)\right]^M
\end{align*}
\end{lemma}
\begin{proof}
By the fact of Normal distribution,
$
  \prob \left( \eta_j \le w \right) = \prob ( { \eta_j } / { \sigma } \le {w} / {\sigma} ) = \Phi({w} / {\sigma}).
$
We have
\begin{align*}
  \prob( \max_{j \in [M]} \eta_j \ge w ) = 1 - \prob( \max_{j \in [M]} \eta_j \le w ) = 1 - \prob( \forall j \in [M], \eta_j \le w )
  = 1 - \left[\Phi\left(\frac{w}{\sigma} \right)\right]^{M}.
\end{align*}
\end{proof}
\begin{proposition}
Let the UCB sequence be 
\[
    U = ( \mu_t(a, b) + \sqrt{\beta'_t} \sigma_t(a, b) : t \in \NN).
\]
Let
\[
M = \frac{\log(\sqrt{t})}{\log \frac{1}{ \Phi(\sqrt{\beta'_t}) } }
\]
and we have,
\begin{align*}
    \prob( \neg \event{E}_{t}(\tilde{R}, U, B_t) ) \le \frac{1}{\sqrt{t}}.
\end{align*}
The summation of failure probability is 
\begin{align*}
    \sum_{t = 1}^T \frac{1}{\sqrt{t}} = 2 \sum_{t = 1}^T\frac{t - (t-1)}{ \sqrt{t} + \sqrt{t} } 
    \le 2 \sum_{t = 1}^T\frac{t - (t-1)}{ \sqrt{t} + \sqrt{t-1} } = 2 \sum_{t=1}^T (\sqrt{t} - \sqrt{t-1}) = 2 \sqrt{T}
\end{align*}
\end{proposition}
\begin{proof}
Recall that the optimistic Thompson sampling estimator $\tilde{R}$ is generated by
\begin{align*}
  \tilde{f}^{\operatorname{OTS}}_{t+1}(a, B_t) := (\max_{j \in [M]} z_{t+1}^j) \cdot \sqrt{\beta'_t} \sigma_{t}(a, B_t) + \mu_{t}(a, B_t)
  ~\text{and}~\tilde{R}_{t+1}(a) = \clip_{[-c, c]} \left( \tilde{f}_{t+1}^{\operatorname{OTS}}(a, B_t) \right), \forall a \in \actions.
\end{align*}
\begin{align}
  \prob( \event{E}_{t}(\tilde{R}, U, B_t) \mid H_t, B_t )
  & = \prob( \tilde{R}_{t+1}(a) \ge U_{t}(a, B_t), \forall a \in \mathcal{A} \mid H_t, B_t) \nonumber \\
  & \stackrel{(i)}{\ge} \prob( \tilde{R}_{t+1}(a) \ge \min\{ \mu_t(a, B_t) + \sqrt{\beta'_t} \sigma_t(a, B_t), c \}, \forall a \in \mathcal{A}  \given H_t, B_t) \nonumber \\
  & \stackrel{(ii)}{\ge} \prob \left( \left(\max_{j \in [M]} z^j_{t+1} \right) \sigma_t(a, B_t) \ge \sqrt{\beta'_t} \sigma_t(a, B_t), \forall a \in \mathcal{A} \given H_t, B_t \right) \nonumber \\
  & = \prob \left( \max_{j \in [M]} z^j_{t+1} \ge \sqrt{\beta'_t} \right) \nonumber \\
  & \stackrel{(iii)}{=} 1 - \Phi\left( \sqrt{\beta_t'} \right)^M .
\end{align}
where $(i)$ is by the definition of $\event{E}(c)$ and $\mathcal{F}_t$ and the fact $B_t$ is conditionally independent of $f_{\theta}$ given $H_t$;
$(ii)$ is due to $\clip_{[-c, c]}( x ) \ge \min(x, c)$ and the function $min(x, c)$ is non-decreasing function on $x$;
$(iii)$ is by the fact $\{ z^j_{t+1} \}_{j \in [M]}$ is independent of $H_{t+1} = (H_t, A_t, B_t, R_{t+1, A_t, B_t})$, the fact that $\sigma_t(a, b)$ and $w_t(a, b)$ are deterministic given $H_t$ and \cref{lem:normal_optimism}.
Solve $\Phi(\beta'_t)^M = 1/\sqrt{t}$, we have
\(
    M = \frac{\log(\sqrt{t})}{ \log \frac{1}{\Phi(\sqrt{\beta'_t})} }.
\)
\end{proof}
\begin{lemma}[Anti-concentration property of maximum of Gaussian R.V.]
\label{lem:anti}
  Consider a normal distribution $N \left(0, \sigma^2 \right)$ where $\sigma$ is a scalar.
  Let $\eta_{1}, \eta_{2}, \ldots, \eta_{M}$ be $M$ independent samples from the distribution.
  Then for any $\delta>0$,
  $$
    \prob \left(\max _{j \in[M]} \eta_{j} \leq \sqrt{ 2 \sigma^2 \log (M/ \delta)}\right) \geq 1- \delta .
  $$
\end{lemma}
\begin{proposition}
    Let the UCB sequences be 
\[
    U' = ( \mu_t(a, b) + \sqrt{2\log (M \sqrt{t})} \sigma_t(a, b): t \in \NN).
\]
Then
\begin{align*}
    \prob( \neg \event{E}^c_t(\tilde{R}, U', A_t, B_t) ) \le \frac{1}{\sqrt{t}}.
\end{align*}
\end{proposition}
\begin{proof}
    This is a direct consequence of \cref{lem:anti} under conditional probability given $H_t, A_t, B_t$ with setting $\delta = 1/\sqrt{t}$.
\end{proof}

\subsection{Bounding estimation regret via information-theoretic quantity}
\label{sec:estimation-information}

To characterize the property of game environments and how much information algorithm can acquire about the environment at each round, we define the information ratio of algorithm $\pi = (\pi_t)_{t\in \NN}$,

From \cref{eq:width}, we have 
\begin{align}
  \beta_t I_t( \theta; R_{t+1, A_t, B_t} \given A_t= a, B_t = b )
   \ge \beta'_t \sigma^2_t(a, b).
\end{align}
where
\[
\beta_{t} = 
\frac{2 \beta'_t}{\log (1 + \sigma_w^{-2} ) }.
\]
Recall that
\[
    U' = ( \mu_t(a, b) + \sqrt{2\log (M \sqrt{t})} \sigma_t(a, b): t \in \NN).
\]
Then,
\begin{align*}
     (U'_t(A_t, B_t) - L_t(A_t, B_t) ) &\le \left(\frac{\sqrt{2 \log (M \sqrt{t} ) } }{ \sqrt{\beta'_t} } + 1 \right) \sqrt{\beta'_t} \sigma_t(A_t, B_t) \\
    & \le \left(\frac{\sqrt{2 \log (M \sqrt{t} ) } }{ \sqrt{\beta'_t} } + 1 \right) \sqrt{\beta_t I_t( \theta; R_{t+1, A_t, B_t} \given A_t, B_t )} \\
    & = \left(\frac{\sqrt{2 \log (M \sqrt{t} ) } }{ \sqrt{\beta'_t} } + 1 \right) \sqrt{ \beta_t I_t( \theta; A_t, B_t, R_{t+1, A_t, B_t} ) }
\end{align*}
where $\frac{\sqrt{2 \log M \sqrt{t} } }{ \sqrt{\beta'_t} } = 
    \sqrt{ \frac{2 \log M \sqrt{t} }{ \log \actions \sqrt{t} } }$.
Then immediately from \cref{prop:general-regret-bound} and Cauchy–Bunyakovsky–Schwarz inequality,
\begin{align*}
  \E{ \sum_{t=0}^{T-1} U'_t(A_t, B_t) - L_t(A_t, B_t) } 
  & \le 
  \E{ \sum_{t=0}^{T-1} \sqrt{  \left(\sqrt{2 \log (M \sqrt{t}) / \beta'_t } + 1 \right)^2 \beta_t I_{t} (\theta ; A_t, B_t, R_{t+1, A_t, B_t} ) } } \\
  & \le \sqrt{ \E{  \sum_{t=0}^{T-1} \left(\sqrt{2 \log (M \sqrt{t}) / \beta'_t } + 1 \right)^2 \beta_t } } \sqrt{\E{  \sum_{t=0}^{T-1} I_{t} (\theta ; A_t, B_t, R_{t+1, A_t, B_t} ) } }
\end{align*}
Let $Z_t = (A_t, B_t, R_{t+1, A_t, B_t}))$, then
\begin{align*}
  \E { \sum_{t=0}^{T-1} I_{t} (\theta ; A_t, B_t, R_{t+1, A_t, B_t} ) }
  & = \E{ \sum_{t} I_t(\theta ; Z_t)} = \sum_{t=0}^{T-1} I( \theta ; Z_t \given Z_0, \ldots, Z_{t-1} ) \\
  & = 
  I( \theta;  Z_0, \ldots, Z_{T-1} ) = I(\theta; H_T)
\end{align*}
Then,
\begin{align*}
  \E{ \sum_{t=0}^{T-1} U'_t(A_t, B_t) - L_t(A_t, B_t) } 
  & \le \sqrt{ T \left(8 \log (M \sqrt{T}) / \log( 1 + \sigma_w^{-2} )  + 2\beta_T \right) I(\theta; H_T)}
\end{align*}
Plugin the choice of $M = \frac{\log(\actions \sqrt{T})}{\log \frac{1}{\Phi(\beta'_t)} }, \beta'_t = 2 \log \actions \sqrt{T}$ and $\beta_t = 2 \beta'_t / \log (1 + \sigma_w^{-2})$,
then we obtain the final result,
\[
    \log (M \sqrt{T}) / \log( 1 + \sigma_w^{-2} ) + 2\beta_T = \mathcal{O} ( \log \actions T + \log T + \log \log \actions T ).
\]
That is
\begin{align}
\label{eq:estimation-error-upper-bound}
  \E{ \sum_{t=0}^{T-1} U'_t(A_t, B_t) - L_t(A_t, B_t) } 
  & \le \mathcal{O} \left(\sqrt{ T \log (\mathcal{A} T) I(\theta; H_T)} \right).
\end{align}

\subsubsection{Bounds on the information gain $I(\theta; H_T)$.}
\begin{remark}
  An important property of the Gaussian distribution is that the information gain does not depend on the observed rewards. This is because the posterior covariance of a multivariate Gaussian is a deterministic function of the points that were sampled. For this reason, this maximum information ratio $\gamma_T$ in \cref{def:max_infor_gain} is well defined. That is,
  $I(\theta; H_T) = I( \theta;  A_0, B_0, \ldots, A_{T-1}, B_{T-1}) \le \gamma_T.$
\end{remark}

We adopt the results from~\citep*{srinivas2009gaussian} which gives the bounds of $\gamma_T$ for a range of commonly used covariance functions: finite dimensional linear, squared exponential and Matern kernels.
\begin{example}[Finite dimensional linear kernels]
  Finite dimensional linear kernels have the form $k(x, x') = x^\top x'$.
  GPs with this kernel correspond to random linear functions $f(x) = \theta^\top x, \theta \sim N(0, \sigma_0 I)$.
\end{example}

\begin{example}[Squared exponential kernel]
  The Squared Exponential kernel is $k(x, x') = \exp( - (2l^2)^{-1} \norm{x - x'}^2s )$, $l$ is a lengthscale parameter. Sample functions are differentiable to any order almost surely.
\end{example}

\begin{example}[Matern kernel]
  The Matern kernel is given by
  $k(x, x') = (2^{1-\nu} / \Gamma(\nu) ) r^\nu B_{\nu}(r)$ 
  and $r = ( \sqrt{2\nu} / l ) \norm{x - x'}$, where $\nu$ controls the smoothness of sample paths (the smaller, the rougher) and $B_{\nu}$ is a modified Bessel function.
  Note that as $\nu \rightarrow \infty$, appropriately rescaled Matern kernels converge to the Squared Exponential kernel.
\end{example}

{
\renewcommand{\arraystretch}{1.4}
\begin{table}[htbp]
  \centering
  \caption{Maximum information gain $\gamma_T$.}
  \begin{tabular}{@{}cccc@{}}
  \toprule
  Kernel     & Linear                    & Squared exponential               & Materns ($\nu > 1$)                                            \\ \midrule
  $\gamma_T$ & $\mathcal{O}( d \log T )$ & $\mathcal{O}( ( \log T )^{d+1} )$ & $\mathcal{O}( T^{d(d+1) /(2 \nu+d(d+1))}(\log T) )$ \\ \bottomrule
  \end{tabular}
\end{table}
}

\section{Experiments}
\def\red{\color{red}}

\subsection{Performance metric} 
\label{appexp:metric}
Two performance metrics are utilized to evaluate the algorithms: average regret and KL divergence to Nash equilibrium.
\paragraph{Average regret:} The expected regret of player $i$ over $T$ time steps is defined as:
    \begin{equation}\label{equ: metric}
        \text{Regret}^{i}(T)=\frac{1}{T}\max _{a \in \Delta^{\mathcal{D}(\mathcal{A}^i)}}\mathbb{E}\left[\sum_{t=1}^{T}  \phi\left(a, x_t^{-i}\right)-\phi\left(x_t^{i}, x_t^{-i}\right)\right],
    \end{equation} 
where $a$ and $x_t^{i}$ represent the strategy for player $i$ and $x_t^{-i}$ denotes the strategies of all other players. $\mathcal{D}(\mathcal{A}^i)$ is the action set for player $i$. The expectation is taken over the randomness of the algorithm and the environment.  

\paragraph{Duality Gap $\&$ KL divergence to Nash:} 
The duality gap for a strategy pair $(x,y)$ is defined as 
\begin{equation*}
    \text{Gap}(x, y)=\max _{\left(x^{\prime}, y^{\prime}\right) \in \Delta} \mathbb{E}\left[\phi\left(x^{\prime}, y\right)-\phi\left(x, y^{\prime}\right)\right] 
\end{equation*}
The duality gap provides a measure of how close a solution pair is to a Nash equilibrium. If a solution pair $(x, y)$ has a duality gap of $\epsilon$, it is considered an $\epsilon$-Nash Equilibrium.

After $T$ iterations, the average-iterate strategies are defined as:
\begin{equation*}
\overline{x}_T=\frac{1}{T} \sum_{t=1}^{T} x_{t}, \quad \overline{y}_{T}=\frac{1}{T} \sum_{t=1}^{T} y_{t}.
\end{equation*}
The KL divergences $KL(\overline{x}_T, x^)$ and $KL(\overline{y}_T, y^)$ are also used as performance metrics for comparing the average-iterate solution pair $(\overline{x}_T,\overline{y}_T)$ with a Nash equilibrium pair $(x^,y^)$.

\subsection{Different opponents in matrix game}
\label{appexp:opponent} 
Four types of opponents in random matrix games are introduced, and the performance of different algorithms against these opponents is compared.

\paragraph{Self-play opponent:} The opponent uses the same algorithm as the player.

\paragraph{Best-response opponent:} The strategy for a best-response opponent is defined as:
\begin{equation} 
    y^*=\argmin\limits_{y\in\Delta}y^T(Ax),
\end{equation}   
which implies that the opponent knows matrix $A$ and the player's strategy $x$ at each round.

\paragraph{Stationary opponent:}The stationary opponent always samples an action from a fixed strategy.

\paragraph{Non-stationary opponent:}A non-stationary opponent changes its strategy randomly every round. 

\subsection{Devices}
Intel(R) Xeon(R) God 6230R CPU @ 2.10GHz and RTX A5000.

\subsection{Additional results for random matrix games}\label{appexp:addres} 
This section presents additional evaluations of different algorithms on two-player zero-sum matrix games. The experiments consider payoff matrices where each entry is an i.i.d. random variable generated from the uniform distribution $[-1,1]$. Each player has $M$ actions, resulting in a squared payoff matrix of size $M$. The total number of rounds is set to $T=10^7$. In each round, the players receive noisy rewards $\pm \Tilde{r}_t$, where $\Tilde{r}t=A{ij}+\epsilon_t$, and $\epsilon_t\sim\mathcal{N}(0, 0.1)$. The experiments investigate different matrix sizes, specifically $M=5,10,20,50,70,100$, and for each choice of $M$, $100$ independent simulation runs are conducted. The performance of the algorithms is averaged over these simulation runs. 

\paragraph{Self-play opponent} 
First, we compare different algorithms under the self-play setting, where both players employ the same algorithm. Convergence curves of two performance metrics (see \cref{appexp:metric}) are shown in \cref{fig:rdnmtx_selfplay}, where each subplot (a)-(f) corresponds to a different matrix size $M$. The results indicate that algorithms exploiting the game structure outperform the two IWE baselines, particularly for smaller matrix sizes. As the matrix dimension increases, the performance gap between the proposed algorithms and the baselines diminishes. Among the algorithms, those based on the OTS method exhibit faster convergence than the ones based on UCB. Additionally, the average regret of RM decreases earlier than that of Hedge. 
 
\begin{figure}
     \centering
     \begin{subfigure}[b]{.49\linewidth}  
         \centering
         \includegraphics[width=\linewidth]{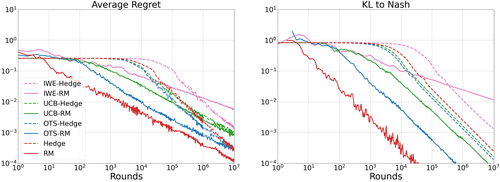} 
         \caption{$5\times5$}
     \end{subfigure}
     \hfill
     \begin{subfigure}[b]{.49\linewidth}  
         \centering
         \includegraphics[width=\linewidth]{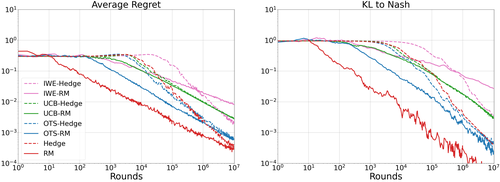}
         \caption{$10\times10$}
     \end{subfigure}
     \hfill
     \begin{subfigure}[b]{0.49\linewidth} 
         \centering
         \includegraphics[width=\linewidth]{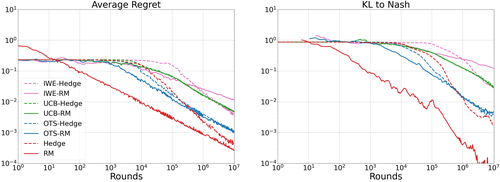}
         \caption{$20\times20$}
     \end{subfigure}
     \hfill
     \begin{subfigure}[b]{0.49\linewidth} 
         \centering
         \includegraphics[width=\linewidth]{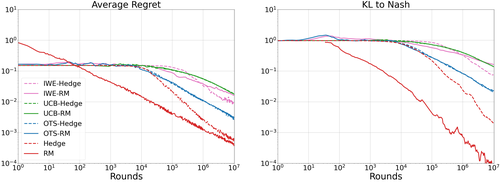}
         \caption{$50\times50$}
     \end{subfigure}
     \hfill
     \begin{subfigure}[b]{0.49\linewidth} 
         \centering
         \includegraphics[width=\linewidth]{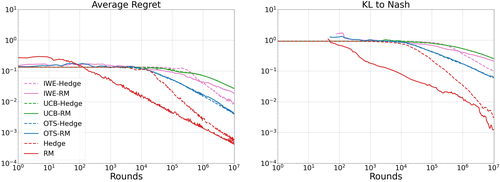}
         \caption{$70\times70$}
     \end{subfigure}
     \hfill
     \begin{subfigure}[b]{0.49\linewidth} 
         \centering
         \includegraphics[width=\linewidth]{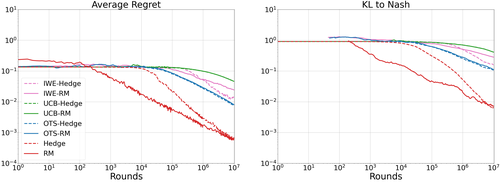}
         \caption{$100\times100$}
     \end{subfigure}
        \caption{Self-play on different random matrix.} 
        \label{fig:rdnmtx_selfplay} 
\end{figure} 

\paragraph{Best-response opponent}
In this subsection, we introduce a best-response opponent (see \cref{appexp:opponent}), while the player continues to use the various algorithms described above. \cref{fig:rdnmtx_br} presents the results for different matrix sizes. We observe that all the algorithms in our proposed framework outperform the IWE baselines. Once again, the OTS-based algorithms demonstrate a faster convergence behavior compared to UCB.

\begin{figure}[htbp]
     \centering
     \begin{subfigure}[b]{.49\linewidth}  
         \centering
         \includegraphics[width=\linewidth]{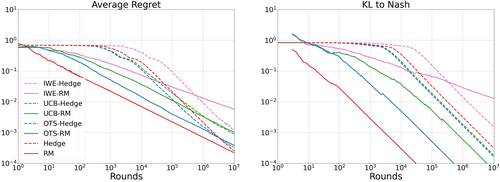}
         \caption{$5\times5$}
     \end{subfigure}
     \hfill
     \begin{subfigure}[b]{.49\linewidth}  
         \centering
         \includegraphics[width=\linewidth]{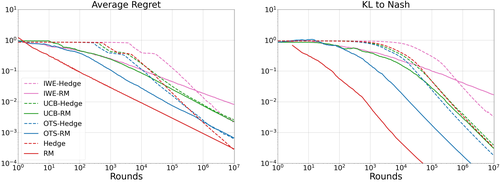}
         \caption{$10\times10$}
     \end{subfigure}
     \hfill
     \begin{subfigure}[b]{0.49\linewidth} 
         \centering
         \includegraphics[width=\linewidth]{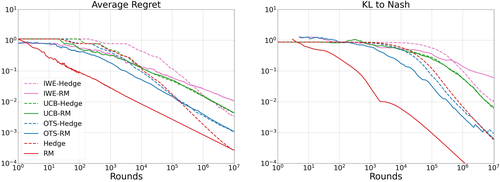}
         \caption{$20\times20$}
     \end{subfigure}
     \hfill
     \begin{subfigure}[b]{0.49\linewidth} 
         \centering
         \includegraphics[width=\linewidth]{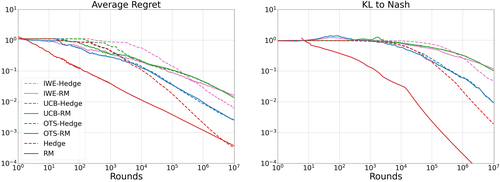}
         \caption{$50\times50$}
     \end{subfigure}
     \hfill
     \begin{subfigure}[b]{0.49\linewidth} 
         \centering
         \includegraphics[width=\linewidth]{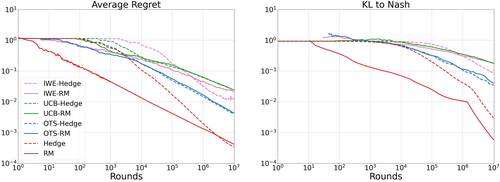}
         \caption{$70\times70$}
     \end{subfigure}
     \hfill
     \begin{subfigure}[b]{0.49\linewidth} 
         \centering
         \includegraphics[width=\linewidth]{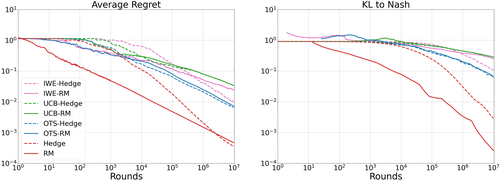}
         \caption{$100\times100$}
     \end{subfigure}
        \caption{Best-response opponent on different random matrix.} 
        \label{fig:rdnmtx_br} 
\end{figure}

\paragraph{Stationary opponent} 
Here, we consider a stationary opponent whose strategy remains fixed as a probability simplex over the action space, with values generated from a uniform distribution. The average regret in this scenario reflects the ability to exploit the opponent's weakness. Convergence curves of the two performance metrics are compared in \cref{fig:rdnmtx_stat}. The results clearly demonstrate that the OTS estimator provides a significant advantage in exploiting this weak opponent, in comparison to the IWE-based estimators.

\begin{figure}[htbp]
     \centering
     \begin{subfigure}[b]{.32\linewidth}  
         \centering
         \includegraphics[width=\linewidth]{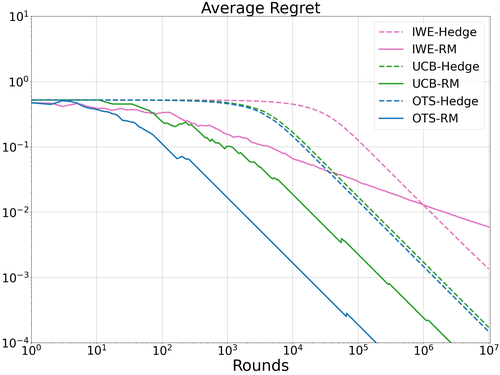}
         \caption{$5\times5$}
     \end{subfigure}
     \hfill
     \begin{subfigure}[b]{.32\linewidth}  
         \centering
         \includegraphics[width=\linewidth]{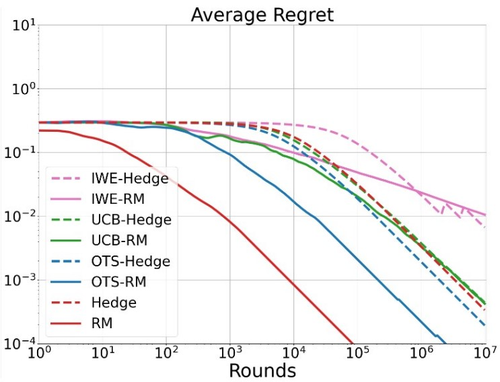}
         \caption{$10\times10$}
     \end{subfigure}
     \hfill
     \begin{subfigure}[b]{0.32\linewidth} 
         \centering
         \includegraphics[width=\linewidth]{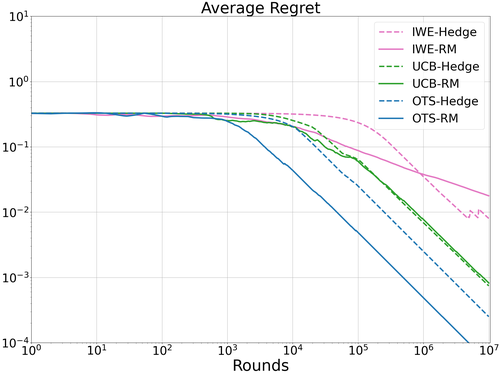}
         \caption{$20\times20$}
     \end{subfigure}
     \hfill
     \begin{subfigure}[b]{0.32\linewidth} 
         \centering
         \includegraphics[width=\linewidth]{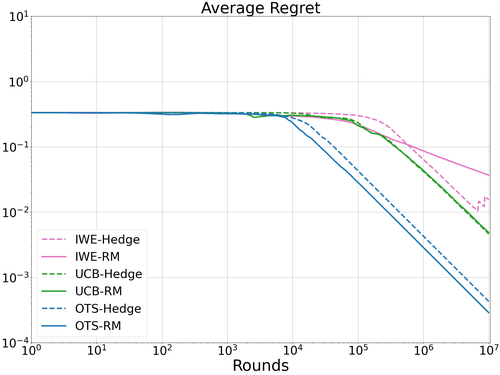}
         \caption{$50\times50$}
     \end{subfigure}
     \hfill
     \begin{subfigure}[b]{0.32\linewidth} 
         \centering
         \includegraphics[width=\linewidth]{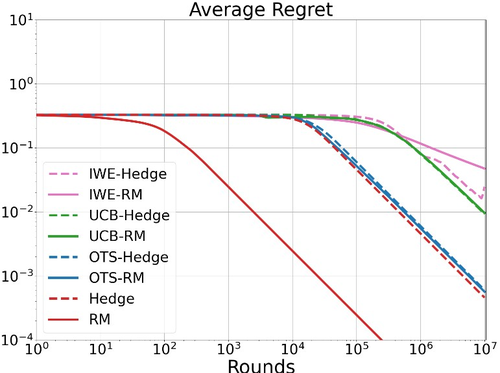}
         \caption{$70\times70$}
     \end{subfigure}
     \hfill
     \begin{subfigure}[b]{0.32\linewidth} 
         \centering
         \includegraphics[width=\linewidth]{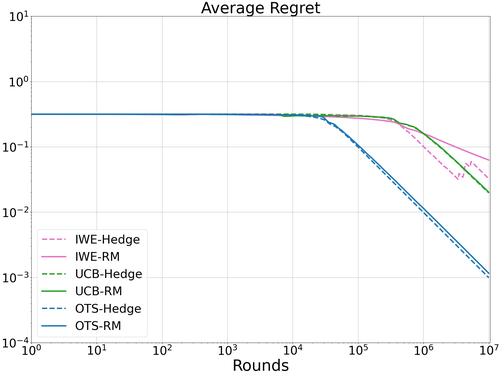}
         \caption{$100\times100$}
     \end{subfigure}
        \caption{Stationary opponent on different random matrix.} 
        \label{fig:rdnmtx_stat} 
\end{figure}

\paragraph{Non-stationary opponent} 
In this subsection, we introduce a non-stationary opponent, requiring the player to develop a robust strategy against all possible opponent's strategies. Specifically, a game matrix $A\in \mathbb{R}^{10\times5}$ is generated, with each element sampled from $\mathcal{N}(0.5, 2.0)$. The opponent's actions are drawn from a fixed strategy that randomly changes every $50$ rounds. Each algorithm is evaluated over $1000$ rounds, and $100$ simulation runs are conducted. \cref{fig:rdnmtx_robust} presents histograms of rewards over all rounds and simulation runs, while \cref{tab:rdnmtx_robust} summarizes the percentage of negative rewards and the mean reward values. The results show that OTS has a smaller percentage of negative rewards and achieves higher mean rewards compared to IWE, indicating its superior robustness. 
\begin{figure} 
    \centering
    \includegraphics[width=1\textwidth]{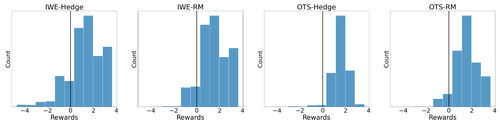} 
    \caption{Reward histograms of different algorithms against a non-stationary opponent.} 
    \label{fig:rdnmtx_robust} 
\end{figure} 

\begin{table}[htbp]
\centering
\begin{tabular}{ccc}
\hline
          & return$<0$ & mean return \\ \hline
IWE-Hedge & 19.4\% & 1.24        \\
IWE-RM    & 12.6\% & 1.50        \\
OTS-Hedge & 2.5\%  & 1.55        \\
OTS-RM    & 8.8\%  & 1.55      \\ \hline 
\end{tabular}
\captionof{table}{Returned rewards of IWE and OTS against a non-stationary opponent.} 
\label{tab:rdnmtx_robust}
\end{table}

\subsection{Sensitivity test for UCB-RM and OTS-RM} \label{appexp:sensitivity} 
In this subsection, we have conducted new experiments involving the selection of different parameters within the UCB-Regret Matching$(\beta)$ and OTS-Regret Matching$(M)$. These experiments encompass different matrix sizes $(10$, $20$, and $50)$ under the self-play setting. The choices of $\beta$ are as follows, presented in descending order: $\sqrt{2\log\mathcal{A}\sqrt{T}}, 0.2\sqrt{2\log\mathcal{A}\mathcal{B}\sqrt{T}}, 0.05\sqrt{2\log\mathcal{A}\mathcal{B}\sqrt{T}}$, and$0.01\sqrt{2\log\mathcal{A}\mathcal{B}\sqrt{T}}$. Notably, $\sqrt{2\log\mathcal{A}\sqrt{T}}$ stems from our proof, while the subsequent selections are smaller values based on $\sqrt{2\log\mathcal{A}\mathcal{B}\sqrt{T}}$. The results demonstrate that decreasing the value of $\beta$ yields an initial decrease followed by an increase in average regret. Notably, irrespective of $\beta$ variation, the performance of OTS-RM is always better than UCB-RM. Details are shown in Figure\ref{fig:beta}. 
\begin{figure}
    \centering
    \includegraphics[width=0.99\linewidth]{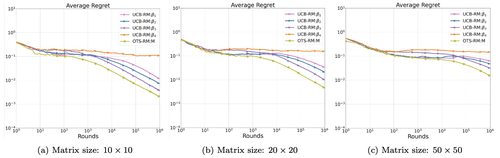}
    \caption{Comparison between UCB-RM and OTS-RM with different choices of $\beta$}
    \label{fig:beta}
\end{figure}

Another experiment on testing the sensitivity of OTS-RM is in Figure\ref{fig:M}. The parameter in OTS-RM is denoted as $M$, which signifies the number of samples taken for optimism in each round. Details about $M$ are outlined in \cref{thm:ots_advantages} and Appendix\ref{sec:estimation-information}. Noteworthy from the observations in Figure\ref{fig:M} is the minimal impact of varying $M$ across different values $(M=10, 20, 30)$ on the performance outcome, suggesting adjustments in $M$ do not significantly influence the performance.
\begin{figure}
    \centering
    \includegraphics[width=0.8\linewidth]{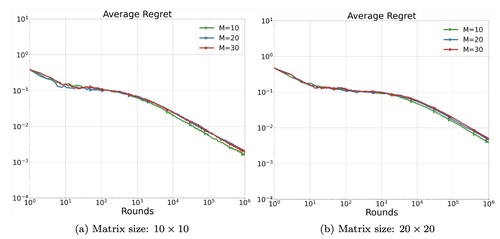}
    \caption{Sensitivity test for OTS-RM with different $M$}
    \label{fig:M}
\end{figure}

\subsection{Convergence rate related with dimensions}\label{appexp:conv}
The average regret bounds for IWE-Hedge and the proposed algorithms are $\Tilde{\mathcal{O}}(\sqrt{(M+N)/T})$ and $\Tilde{\mathcal{O}}(\sqrt{MN/T})$, respectively. Taking OTS-Hedge as an example, for a fixed iteration $T$, IWE-Hedge implies $\log(\text{average regret})\propto (1/2)\log(M+N)$, while OTS-Hedge indicates $\log(\text{average regret})\propto \frac{1}{2}\log(MN)$. When $M=N$, the logarithmic average regret for IWE-Hedge and TS-Hedge should increase with respect to $M$ at the rates of $1/2$ and $1$, respectively. The experimental results shown in \cref{fig:rndmtx_rate}, where $M$ ranges from $2$ to $100$ with a fixed $T=10^7$, match the theoretical predictions almost exactly. These empirical results provide strong support for our regret analysis. 

\begin{figure}
    \centering
    \includegraphics[width=0.7\linewidth]{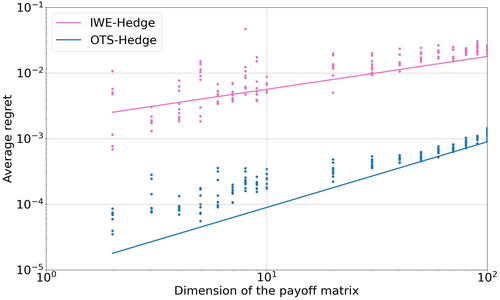} 
    \caption{Convergence rate with different dimensions.} 
    \label{fig:rndmtx_rate}
\end{figure}

\subsection{Radar anti-jamming problem} 
\label{appexp:radar}
The competition between radar and jammer is an important issue in modern electronic warfare, which can be viewed as a non-cooperative game with two players. This competition occurs at the signal level, where both the radar and jammer can change parameters of their transmitted signals. One representative game form is playing in the frequency domain, as the signal with different carrier frequencies is disjoint. 

In our example, we consider the pulse radar and the noise-modulated jammer. The radar transmits pulse signals one by one, with a waiting time interval between consecutive pulses. At the beginning of each pulse, both the radar and jammer transmit their own signals. After a short signal propagation time, each player receives their opponent's signal and obtains a reward. In our setting, both the radar and jammer have three candidate carrier frequencies, denoted as $\mathcal{F}=\{f_1, f_2, f_3\}$. The radar player has three sub-pulses in each radar pulse, and each sub-pulse can choose a different carrier frequency. The action set of the radar is denoted as $\mathcal{A}_R = \mathcal{F} \times \mathcal{F} \times \mathcal{F}$, which has a total of $27$ different actions. On the other hand, the jammer player can choose one carrier frequency to transmit the jamming signal and change the carrier frequency for different radar pulses. The action set of the jammer is denoted as $\mathcal{A}_J = \mathcal{F}$. After each iteration between the radar and jammer, the radar obtains a signal-inference-and-noise ratio (SINR), which serves as the reward for the radar in that round.

We can observe that both the radar and jammer's actions are related to the frequency set $\mathcal{F}=\{f_1, f_2, f_3\}$. Therefore, the reward in each round can be further defined as a linear function, making the anti-jamming scenario a linear game for the radar. 
\begin{definition}
The reward function in the anti-jamming problem is defined as follows:
$$\text{SINR}(a,b;\theta)=\phi(a,b)^T\theta$$ 
Here, $\phi(a,b)=P_a(\theta)/(P_{n_0}+P_{b}\mathds{1}(a=b))$ represents a known feature mapping that maps the actions of the radar and the jammer, denoted as $(a,b)$, into the frequency domain related with $\theta$. The parameter $\theta$ corresponds to the radar cross section (RCS) associated with the frequencies in the anti-jamming scenario. $P_a(\theta)$ is radar's received power related with $\theta$, $P_{n_0}$ is the noise power and $P_b$ is received jammer's power. The indicator function $\mathds{1}(a=b)$ evaluates to $1$ if the radar and jammer transmit on the same carrier frequency $(a=b)$ and $0$ otherwise. 
\end{definition}
This definition captures the essence of the anti-jamming problem, allowing us to evaluate the performance of different strategies and algorithms based on the signal-to-interference-plus-noise ratio (SINR), and construct the anti-jamming scenario as a linear game. 
\begin{figure}
    \centering
    \includegraphics[width=.5\linewidth]{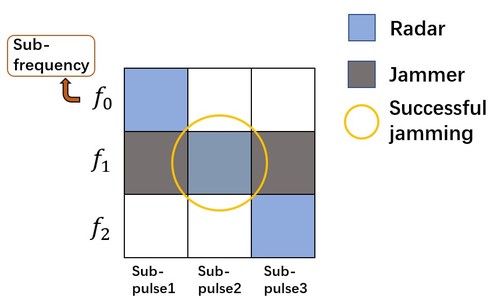}
    \caption{Illustration of the game between radar and jammer for one pulse.}
    \label{fig:radar02}
\end{figure}

\begin{table}
    \centering
    \caption{Parameters of FA radar and jammer.} 
    \begin{tabular}{cc} 
    \hline
    Parameter & Value \\
    \hline
    radar transmitter power $P_T$ & $30$kW \\
    radar transmit antenna gain $G_T$ & $32$dB  \\
    radar initial frequency $f_0$ & $3$GHz \\
    bandwidth of each subpulse $B$ & $2$MHz \\
    distance between the radar and the jammer & $100$km \\
    false alarm rate $p_f$ & $1\times10^{-4}$ \\
    jammer transmitter power $P_J$ & $100$W \\
    jammer transmit antenna gain $G_J$ & $15$dB \\ 
    \hline 
    \end{tabular} 
    \label{tab:radar}
\end{table}

\paragraph{Adaptive jammer} 
In the case of an adaptive jammer, it takes actions according to the radar's latest 10 pulses. Specifically, it counts the numbers of different carrier frequencies ($f_1, f_2, f_3 \in \mathcal{F}$) appearing in the radar's last 10 pulses, denoted as $N_1, N_2$, and $N_3$. The jammer then takes action $f_i$ with a probability proportional to $N_i$ for $i=1,2,3$.

\subsection{Repeated traffic routing problem}\label{app:traffic}
This section considers the traffic routing problem in the transportation literature, which is defined over a directed graph $(\mathcal{V},\mathcal{E})$ ($\mathcal{V}$ and $\mathcal{E}$ are vertices and edges sets respectively) and modeled as a multi-player game. Each node pair (referred as one origin node and one destination node) in the graph is treated as an individual player and every player $i$ seeks to find the `best' route (consists of several edges) to send $U^i$ units from its origin node to its destination node. The quality of the chosen route is measured by the travel time, which depends on the total occupancy of the traversed edges. If one edge is occupied by more players, more travel time of this edge is. Specifically, the travel time $t_e$ of edge $e\in\mathcal{E}$ is a function of the total units $u$ traversing $e$. One common choice is the Bureau of Public Roads function~\cite{leblanc1975algorithm}
$$
t_e(u)=c_e(1+0.5(\frac{u}{C_e})^4),
$$
where $c_e$ and $C_e$ are free-flow and capacity of edge $e$ respectively. The action of player $i$ is denoted as $a^i\in\mathcal{A}^i\subset \mathbb{R}^{|\mathcal{E}|}$ and the component corresponding to edge $e$ is denoted as $[a^i]_e$. If edge $e$ belongs to the route, $[a^i]_e=U^i$ and otherwise $[a^i]_e=0$. Further, let $a^{-i}\in\mathcal{A}^{-i}$ be the action of other players, the total occupancy of edge $e$ is $[a^i]_e+[g(a^{-i})]_e$, where $g(a^{-i})=\sum_{j\neq i}a^j$. This way, the total travel time of a joint action $\bm{a}=(a^i,a^{-i})$ for player $i$ can be expressed as 
$$
\ell^i(a^i, a^{-i})=\sum_{e\in\mathcal{E}}[a^i]_e t_e([a^i]_e+[g(a^{-i})]_e).
$$
The reward function of player $i$ is $r^i(a^i, a^{-i})=-\ell^i(a^i, a^{-i})$. Note that the mathematical form of the reward function is unknown to players, only values of $r^i(a^i, a^{-i})$ and actions $(a^i,a^{-i})$ can be observed by players.

In our experiment, the Sioux-Falls road network data set~\cite{bar2015transportation} is used and we set $c_e$ and $C_e$~\cite{bar2015transportation}. This network is a directed graph with $24$ nodes and $76$ edges and there are in total $N=528$ players. Each player $i$'s action space $\mathcal{A}^i$ is specified by the 5 shortest routes and any route that more than three times longer than the shortest route is further removed from $\mathcal{A}^i$. To exploit the correlations among actions $(a^i, a^{-i})$ in the reward function, the composite kernel proposed in~\cite{sessa2019no} is used. For player $i$, let $a^i,b^i\in\mathcal{A}^{i}$ and $a^{-i},b^{-i}\in\mathcal{A}^{-i}$, then the kernel function associating $(a^i, a^{-i})$ and $(b^i, b^{-i})$ is 
$$
K^{i}((a^i, a^{-i}),(b^i, b^{-i}))=k_{\rm L}(a^i, b^i)k_{\rm P}(a^i+g(a^{-i}), b^i+g(b^{-i})),
$$
where $k_{\rm L}(\cdot,\cdot)$ and $k_{\rm P}(\cdot,\cdot)$ are linear and polynomial kernels respectively. The hyperparameters of kernels are set the same as in \cite{sessa2019no}. Details on GP for estimating reward functions can be found in Section 2 of \cite{sessa2019no} and we do not repeated them here\footnote{  The implementation refers the code released by authors of \cite{sessa2019no} at \url{https://github.com/sessap/stackelucb}}. Except the regret, the congestion of edge $e$ is also included as a performance metric for a joint action $\bm{a}$, which is computed as $0.15(\sum_{i=1}^N [a^i]_e/C_e)^4$. The averaged congestion is obtained by averaging congestion over all edges.

\end{document}